\documentclass{article}

\PassOptionsToPackage{numbers,sort&compress}{natbib}
\usepackage[final,main]{neurips_2026}

\usepackage[utf8]{inputenc}
\usepackage[T1]{fontenc}
\usepackage{microtype}
\usepackage{graphicx}
\usepackage{booktabs}
\usepackage{hyperref}
\usepackage{url}
\usepackage{amsfonts}

\usepackage{amsmath}
\usepackage{amssymb}
\usepackage{mathtools}
\usepackage{amsthm}
\usepackage{cancel}

\usepackage[dvipsnames]{xcolor}
\usepackage{multirow}
\usepackage{colortbl}
\usepackage{bm}
\usepackage{adjustbox}
\usepackage{listings}
\usepackage{float}
\usepackage{wrapfig}
\usepackage{needspace}
\usepackage{enumitem}
\usepackage{tcolorbox}
\usepackage{nicefrac}
\usepackage{verbatim}
\usepackage[capitalize,noabbrev]{cleveref}
\usepackage{pifont}
\usepackage{algorithm}
\usepackage{algorithmic}

\newif\ifshowacknowledgments
\showacknowledgmentstrue
\newif\ifshowchecklist
\showchecklisttrue

\floatstyle{ruled}
\newfloat{listing}{tb}{lst}
\floatname{listing}{Listing}

\definecolor{anti-flashwhite}{rgb}{0.95, 0.95, 0.96}
\definecolor{backcolour}{rgb}{0.95,0.98,0.98}
\definecolor{codegreen}{rgb}{0,0.6,0}
\definecolor{box_background}{HTML}{FAFAFA}
\definecolor{box_margin}{HTML}{454545}
\definecolor{bestmark}{HTML}{005CE2}
\newcommand{\cunderbrace}[2]{%
  \underbrace{#1}_{#2}%
}
\DeclareMathOperator*{\argmin}{arg\,min}

\theoremstyle{plain}
\newtheorem{theorem}{Theorem}

\theoremstyle{definition}
\newtheorem{definition}{Definition}

\newcommand{\grad}{\nabla}
\newcommand{\boldtheta}{{\bm{\theta}}}
\newcommand{\boldthetahat}{{\hat{\boldtheta}}}

\allowdisplaybreaks

\title{Generalized Influence Functions \\ for Better Model Change Estimates}

\author{%
  Hyeonsu Lyu\textsuperscript{1} \quad
  Jonggyu Jang\textsuperscript{3} \quad
  Sehyun Ryu\textsuperscript{4} \quad
  Hyun Jong Yang\textsuperscript{1,2}\thanks{Corresponding author.}\\[0.5ex]
  {\normalfont\textsuperscript{1}Institute of New Media and Communications, Seoul National University, Seoul, Korea}\\
  {\normalfont\textsuperscript{2}Department of Electrical and Computer Engineering, Seoul National University, Seoul, Korea}\\
  {\normalfont\textsuperscript{3}Chungnam National University, Daejeon, Korea}\\
  {\normalfont\textsuperscript{4}Pohang University of Science and Technology, Pohang, Korea}\\[0.5ex]
  \texttt{hs.lyu@snu.ac.kr} \enspace
  \texttt{jgjang@cnu.ac.kr} \enspace
  \texttt{sh.ryu@postech.ac.kr} \enspace
  \texttt{hjyang@snu.ac.kr}
}

\begin{document}

\maketitle

\begin{abstract}
Influence functions (IFs) approximate the effect of removing or reweighting training data on a trained model, but their linear approximation often lacks the accuracy required for reliable post hoc model editing.
We identify two culprits and propose generalized influence functions (GIFs) and pseudo-LiSSA (p-LiSSA) to address each in turn.
First, classical IFs update all parameters indiscriminately, leading to unnecessary updates to data-irrelevant parameters.
GIFs instead restrict edits to a \textit{data-relevant subspace} and leave the remaining parameters fixed. 
We theoretically show that GIFs achieve a tighter error bound than classical IFs when the parameter subset is properly selected, and that effective subspace selection requires avoiding low-curvature directions while retaining a sufficiently large component of the target gradient.
Second, LiSSA iterations are destabilized by negative Hessian eigenvalues, which existing methods handle via $\ell_2$ damping---at the cost of a biased inverse-Hessian-vector product (iHVP) estimate.
p-LiSSA eliminates this bias by solving the \textit{restricted inverse-Hessian problem} without such damping.
We further show that the presence of low-curvature directions slows p-LiSSA convergence, and that the data-relevant subspace selected by GIFs mitigates this.
Across five image and text unlearning benchmarks, GIFs update only 10\% of the parameters yet preserve accuracy up to 2.12 percentage points better than IF-based baselines, while improving unlearning performance by up to 2.73 percentage points. Moreover, p-LiSSA improves restricted iHVP estimation by reducing the residual by 9.6$\times$ over existing iHVP baselines, while using 19.8\% less memory with comparable computation time.
Code is available at \url{https://github.com/hslyu/GIF}.
\end{abstract}

\section{Introduction}

\begin{wrapfigure}{r}{0.38\textwidth}
    \vspace{-12.5mm}
    \centering
    \includegraphics[width=\linewidth, trim=0 0 0 0, clip]{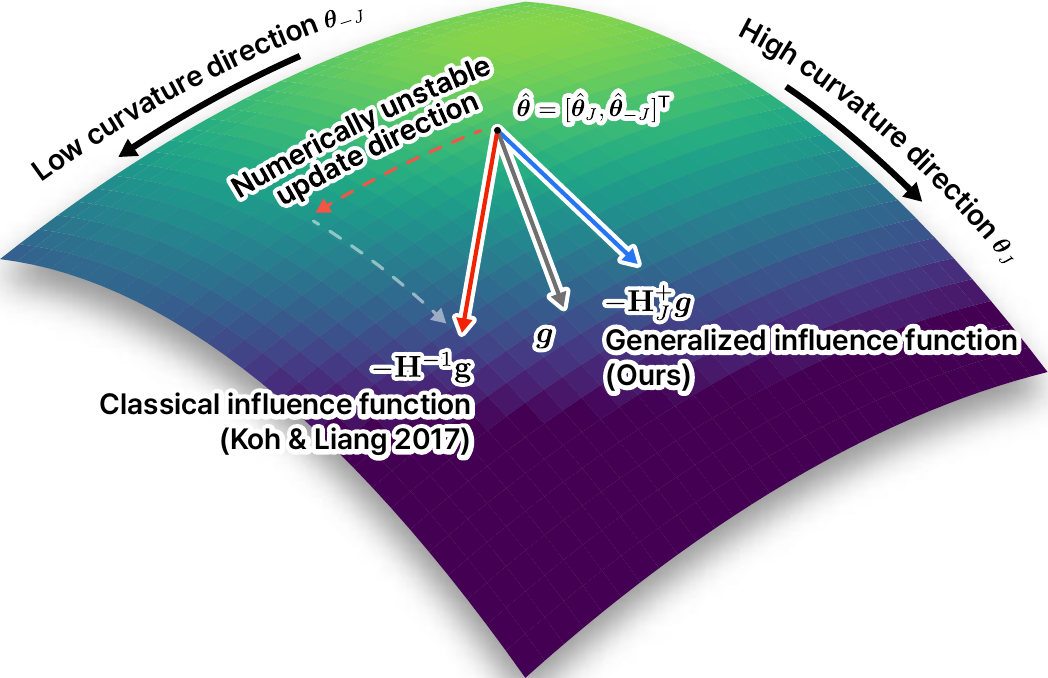}
    \vspace{-3mm}
    \caption{
    Conceptual illustration of classical and generalized IF directions.
    }
    \label{fig:concept}
    \vspace{-3mm}
\end{wrapfigure}

Influence functions (IFs) estimate how a trained model would change if selected training examples were removed, reweighted, or relabeled, by linearizing the effect of a data edit around the current optimum as an inverse-Hessian-vector product (iHVP), $-\mathbf{H}^{-1}\mathbf{g}$, where $\mathbf{H}$ is the Hessian of the training loss and $\mathbf{g}$ is the gradient induced by the selected training examples~\cite{Koh2017_IF,Wu22-PUMA}.
This has motivated their use in post-hoc model editing~\cite{Sinitsin2020-Editable,de2021-editing}, including data removal~\cite{Wu23-GraphUnlearning, Wu22-PUMA} and relabeling~\cite{kong2022-Relabeling, Gong2022-LabelInfluence}.
In practice, however, as shown in previous studies~\cite{Basu2021-fragile,Epifano2023-fragile2}, classical IFs often fail to produce reliable model-change estimates, and have consequently seen more use in data attribution rather than in direct editing~\cite{Koh2017_IF,Pruthi2020-TracIn,schioppa2022-scalingIF,kwon2024datainf,yang2023-dataset,rubinstein2025-rescaled}.

This unreliability stems from the iHVP over the full parameter space, leading to two compounding problems.
The first problem is that the classical IF update $\mathbf{H}^{-1}\mathbf{g}$ can deviate significantly along low-curvature directions, where the inverse Hessian amplifies even small gradient components into disproportionately large parameter changes, as depicted in \cref{fig:concept}.
This full-parameter view is also inconsistent with the observation that neural networks organize behavior through localized units and parameter subsets~\cite{olah2018-interpretability,Bau2020-Understanding}---the iHVP can therefore spread the model update into parameters weakly related to the edited data, causing unintended changes that should have been preserved. 
The second problem is the instability of LiSSA---the de facto standard iterative method to compute iHVP for IFs~\cite{Agarwal2017-HessianEstimator, Koh2017_IF}.
For the convergence of LiSSA iteration, $\mathbf H$ needs to be positive definite; however, this condition can be easily violated in modern neural networks~\cite{Basu2021-fragile,Epifano2023-fragile2}.
A standard workaround is to use $\ell_2$ damping, which 
computes $(\mathbf H+\lambda \mathbf I)^{-1}\mathbf g$ instead of $\mathbf H^{-1}\mathbf g$~\cite{Koh2019_GIF,bae2022-if,Basu20-SecondIF}.
The $\ell_2$ damping stabilizes the iteration but distorts the loss landscape that IFs are meant to approximate.
Moreover, recent work has emphasized that iHVP and curvature approximations can strongly affect the quality of influence estimates \cite{Yegor2024-iHVP,hong2025-better}.
To summarize, classical IFs face two coupled obstacles: \textbf{(i)} IFs change parameters that should be preserved, and \textbf{(ii)} the numerically unstable LiSSA iteration.

\begin{figure*}[t]
    \centering
    \includegraphics[width=\textwidth]{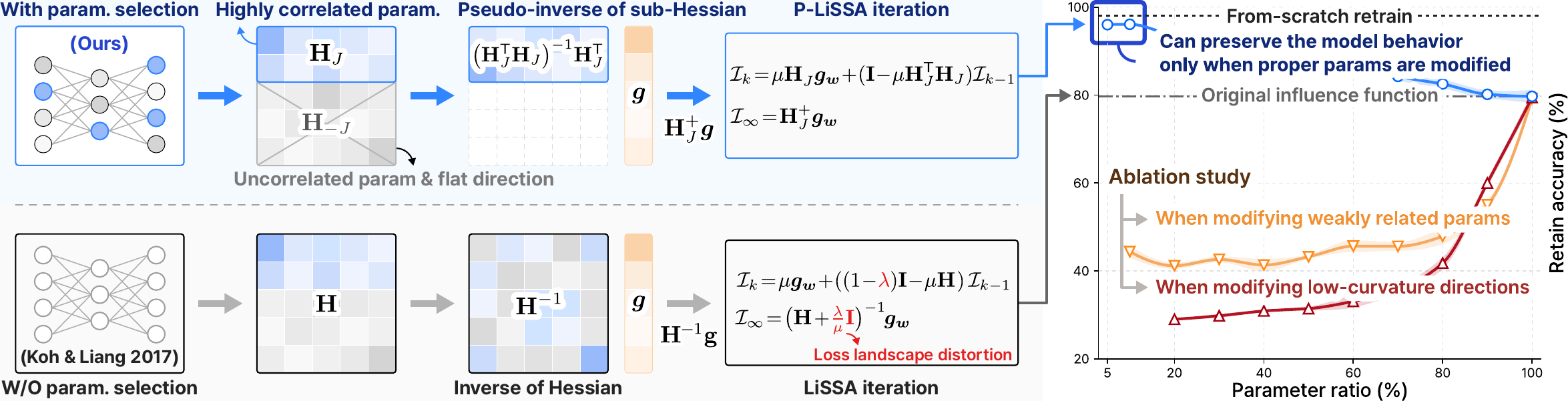}
    \vspace{-5mm}
    \caption{
    Comparison between GIFs and IFs; and between P-LiSSA and LiSSA.
    }
    \label{fig:overview}
    \vspace{-5mm}
\end{figure*}

To address the first problem, we introduce \textit{generalized influence functions (GIFs)}, which decompose the model into target parameters and fixed parameters, then restrict the iHVP computation to the target parameters. 
As shown in the top row of \cref{fig:overview}, GIFs restrict the edit to parameters that are most relevant to the edit data, in contrast to classical IFs (bottom row) that spread the update across all parameters. 
We theoretically show that GIFs offer a tighter approximation error bound than classical IFs when the parameter subset is properly selected, and that effective subspace selection requires avoiding low-curvature directions while preserving a sufficiently large component of the target gradient $\mathbf{g}$ (in \textbf{\cref{theorem:Error bound of the GIF,theorem:Self-loss bound of GIFs}}). Based on these results, we examine a range of parameter-selection criteria and find, as shown in the right panel of \cref{fig:overview}, a criterion favoring data-relevant parameters closely matches the behavior of the from-scratch retrained model even with a small number of parameter updates, while selecting weakly related parameters or low-curvature directions degrades accuracy.


To address the second problem, we propose \textit{pseudo-LiSSA (p-LiSSA)}, a damping-free solver for the restricted inverse sub-Hessian problem induced by GIFs, as shown in \cref{fig:overview}.
By restricting the solve to the target parameter subspace, p-LiSSA operates on a positive-definite Gram matrix instead of the possibly indefinite full Hessian, thereby mitigating instability. 
To support the claim, we theoretically show that p-LiSSA converges to the exact GIF direction under a simple scaling condition without extra damping (in \textbf{\cref{theorem:Convergence of p-LiSSA}}), where the convergence speed can be accelerated by avoiding low curvature.
We further show that p-LiSSA with two independent mini-batches is an unbiased estimator of the truncated Neumann series, and its expectation converges to the GIF direction (in \textbf{\cref{theorem:Unbiasedness of p-LiSSA}}).
We note that p-LiSSA also serves as a damping-free iHVP solver for classical full-parameter IFs, which arise as the special case of GIFs when all parameters are editable.

In experiments, we numerically evaluate whether GIFs improve the model-editing performance in class-level unlearning, where the edited model should forget target data while preserving performance on retained data. 
We then examine various parameter-selection criteria to confirm that data-relevance and curvature-awareness are both essential for reliable model editing. 
Finally, we benchmark p-LiSSA against existing iHVP solvers and show that it reduces the iHVP residual error by 9.6$\times$ over the best competing solver at lower memory cost.

\section{Preliminaries on Classical Influence Functions}
\label{sec:Preliminaries on Finite Influence Updates}

\paragraph{Notations and definitions.}
We define a model $\boldtheta=[\theta_1, ..., \theta_n]^\mathsf{T}\in \mathbb{R}^{n}$ and $m$ training examples $z_1,...,z_m \in \mathcal{Z}$.
Let $\mathcal{L}(\boldtheta) = \frac{1}{m}\sum_{i=1}^m \ell(z_i,\boldtheta)$ be the empirical risk, where $\ell(\cdot)$ is twice continuously differentiable in a neighborhood of the reference minimizer.
The empirical risk minimizer is defined as $\boldthetahat = \argmin_{\boldtheta\in \Theta}\mathcal{L}(\boldtheta)$.
We define the $\epsilon$-weighted empirical risk minimizer as
\begin{align}
    \boldthetahat(\bm{1}+\epsilon\bm{w}) = \argmin_{\boldtheta \in \Theta} 
    \frac{1}{m}\sum_{i=1}^m (1+\epsilon w_i)\ell(z_i,\boldtheta),
    \label{eq:upweighted empirical risk}
\end{align}
where vector $\bm{w} = [w_1, ..., w_m]^\mathsf{T} \in \{0, 1\}^m$ indicates which training examples are edited.

To approximate the parameter change induced by a finite perturbation of the empirical risk, we define the following finite influence update.
\begin{tcolorbox}[
width=1.00\linewidth, colframe=box_margin, colback=box_background, boxrule=0.3mm, boxsep=-3pt, left=8pt, right=8pt]
\begin{definition}[Finite influence function]
    \label{def:IF}
    Let $\mathcal{L}(\boldtheta)$ be twice continuously differentiable in a neighborhood of $\boldthetahat$, and assume $\mathbf{H}_{\bm{1}+\epsilon\bm{w}}$ is invertible.
    For $\Delta\boldthetahat_{\bm{1}+\epsilon\bm{w}}
=
\boldthetahat(\bm{1}+\epsilon\bm{w})-\boldthetahat$, 
    the finite influence function and its parameter displacement are defined as
    \begin{align}
    \mathcal{I}(\bm{w}, \boldthetahat\mid\epsilon)
    \mathrel{\mathop:}=
    -\mathbf{H}^{-1}_{\bm{1}+\epsilon\bm{w}}\bm{g}_{\bm{w}}\in \mathbb{R}^n
    ,\quad    
    \Delta\boldthetahat_{\bm{1}+\epsilon\bm{w}}
    :=
    \epsilon\mathcal{I}(\bm{w}, \boldthetahat|\epsilon)
    =
    -\epsilon\mathbf{H}^{-1}_{\bm{1}+\epsilon\bm{w}}\bm{g}_{\bm{w}},
    \label{eq:finite_linearized_update}
    \end{align}
    where
    $\mathbf{H}_{\bm{1}+\epsilon\bm{w}} = \frac{1}{m}\sum_{i=1}^m(1+\epsilon w_i)\grad^2\ell\big(z_i, \boldthetahat\big)$ and 
    $\bm{g}_{\bm{w}}=\frac{1}{m}\sum_{i=1}^m w_i\grad\ell\big(z_i,\boldthetahat\big)$.
\end{definition}
\end{tcolorbox}

The quantity $\mathcal{I}(\bm{w},\boldthetahat|\epsilon)$ is a parameter-sized update direction induced by the perturbation from $\ell(\cdot)$ to $(1+\epsilon)\ell(\cdot)$ on the selected data points.
This update is obtained by linearizing the first-order optimality condition of \eqref{eq:upweighted empirical risk}, which can be mathematically expressed as
\begin{align}
    0=\grad\mathcal{L}\big(\boldthetahat(\bm{1}+\epsilon\bm{w})\big)
    + \frac{\epsilon}{m}\sum_{i=1}^m w_i\grad\ell\big(z_i,\boldthetahat(\bm{1}+\epsilon\bm{w})\big).
    \label{eq:first order optimality}
\end{align}
Let
$\Delta\boldthetahat_{\bm{1}+\epsilon\bm{w}}
=
\boldthetahat(\bm{1}+\epsilon\bm{w})-\boldthetahat$.
Applying a Taylor expansion to both terms of \eqref{eq:first order optimality} around $\boldthetahat$ gives
\begin{align}
    0
    &=
        \frac{1}{m}
        \sum_{i=1}^m (1+\epsilon w_i)\nabla\ell\big(z_i,\boldthetahat\big)
    +
    \frac{1}{m}\sum_{i=1}^m (1+\epsilon w_i)\nabla^2\ell\big(z_i,\boldthetahat\big)
    \Delta\boldthetahat_{\bm{1}+\epsilon\bm{w}}
    +
    \bm{r}(\epsilon)
    \label{eq:taylor_expansion_1st_order_optimality}
    \\
    &=
    \epsilon\bm g_{\bm w}
    +
    \mathbf H_{\bm 1+\epsilon \bm w}
    \Delta\boldthetahat_{\bm{1}+\epsilon\bm{w}}
    +
    \bm r(\epsilon).
    \label{eq:taylor_derivation_influence}
\end{align}
where $\bm{r}(\epsilon)$ collects higher-order terms.
Neglecting the higher-order remainder gives the linearized approximation
$
    \Delta\boldthetahat_{\bm{1}+\epsilon\bm{w}}
    \approx
    -\epsilon\mathbf{H}^{-1}_{\bm{1}+\epsilon\bm{w}}\bm{g}_{\bm{w}}
    =
    \epsilon\mathcal{I}(\bm{w},\boldthetahat|\epsilon),
$
which is exactly the linearized parameter displacement induced by the finite influence update in \textbf{Definition~\ref{def:IF}}.

We study the finite-step linearized update rather than the infinitesimal derivative at $\epsilon=0$ as in \cite{Koh2019_GIF}.
This choice is motivated by model editing and unlearning, where the goal is to approximate the actual parameter displacement produced by removing or modifying a subset of training points.
Using $\mathbf{H}_{\bm{1}+\epsilon\bm{w}}$ allows the local curvature surrogate to reflect the edited weighting pattern itself, yielding a Newton-style update tailored to the perturbed objective.
This finite formulation is also aligned with the Newton-step approximation used in prior work on group effects and data removal \citep{Koh2019_GIF}.

\paragraph{Numerical vulnerability of LiSSA.} 
\label{Sec:LiSSA approximation}

The authors of \cite{Agarwal2017-HessianEstimator, Koh2017_IF} propose an iterative method,
LiSSA, to numerically compute iHVPs using HVPs.
LiSSA is based on the Neumann expansion \cite{Stewart98-Neumann},
$\mathbf{H}^{-1}=\sum_{i=0}^{\infty}(\mathbf{I}-\mathbf{H})^i$,
which converges if $\rho(\mathbf{H})<1$.\footnote{$\rho(\cdot)$: The spectral radius of a matrix is the largest absolute value among its eigenvalues.}
In practice, however, LiSSA introduces
scaling and damping into the recursion \cite{bae2022-if, Basu20-SecondIF, Koh2019_GIF} as modern neural networks are unlikely to meet the spectral condition as explained in Appendix~\ref{Appendix:The Existence of Negative Eigenvalues}.
With a scaling parameter $\mu>0$ and a damping parameter $\lambda>0$, the LiSSA iteration is implemented as
\begin{align}
    \mathcal{I}_k
    =
    \mathcal{I}_0
    +
    \left(
        (1-\lambda)\mathbf{I}
        -
        \mu
            \mathbf{H}_{\bm{1}+\epsilon\bm{w}}
    \right)
    \mathcal{I}_{k-1},
    \quad
    \mathcal{I}_0=\mu\bm{g}_{\bm{w}},
    \label{eq:approximate_series_IF}
\end{align}
where $\mathcal{I}_k$ converges to 
$
   \big(\mathbf{H}_{\bm{1}+\epsilon\bm{w}} + \frac{\lambda}{\mu} \mathbf{I} \big)^{-1} \bm{g}_{\bm{w}}
    \approx
    -\mathcal{I}(\bm{w},\boldthetahat|\epsilon)
$
only if
$
\rho\left(
    (1-\lambda)\mathbf{I}
    -
    \mu\mathbf{H}_{\bm{1}+\epsilon\bm{w}}
\right)<1
$.
Then, the resulting estimate no longer converges to
$
\mathbf{H}^{-1}_{\bm{1}+\epsilon\bm{w}}\bm{g}_{\bm{w}}
=
-\mathcal{I}(\bm{w},\boldthetahat|\epsilon)
$
, but instead targets a damped surrogate of this direction.
This facilitates numerical convergence, but also distorts the loss landscape that IFs are intended to approximate.
Such distortion is especially undesirable for finite model edits, where accurately approximating the parameter change is more critical than merely estimating data attribution.
\section{Generalized Influence Functions}
\label{sec:Proposed Method: Generalized IF}

We introduce GIFs, which compute influence over selected parameters by separating the model parameters into \textbf{(i) target parameters}, which are editable, and \textbf{(ii) fixed parameters}, which remain unchanged.
We then analyze when GIFs can better approximate retraining behavior than IFs, and study how to select suitable target parameters.
All proofs are provided in the appendices.

\paragraph{Definition and derivation.}
Let $J=\{j_1, ...,  j_k\}$ denote an index set of $k$ target parameters.
We define the target parameters $\boldtheta_J=[\theta_j]_{j\in J}^\mathsf{T} \in \mathbb{R}^k$ and the fixed parameters as $\boldtheta_{-J} = [\theta_j]_{j\notin J} \in \mathbb{R}^{n-k}$.
For brevity, we rearrange the parameters as
$\boldthetahat=[\boldthetahat_J, \boldthetahat_{-J}]^\mathsf{T}$.
Under the same reordering, we write the Hessian as
$
\mathbf{H}_{\bm{1}+\epsilon\bm{w}} = [\mathbf{H}_J \mid \mathbf{H}_{-J}],
$
where $\mathbf{H}_J\in\mathbb{R}^{n\times k}$ collects the columns indexed by $J$ and $\mathbf{H}_{-J}\in\mathbb{R}^{n\times (n-k)}$ collects the remaining columns.
Then, the GIF is formally defined as follows.
\begin{tcolorbox}[
width=1.00\linewidth, colframe=box_margin, colback=box_background, boxrule=0.3mm, boxsep=-3pt, left=8pt, right=8pt]
\begin{definition}[Generalized influence function]
    \label{def:GIF}
    Assume that $\mathbf{H}_J$ has full column rank so that $(\mathbf{H}_J^\mathsf{T}\mathbf{H}_J)^{-1}$ exists.
    Then, for $\mathbf{H}_J^{+} = \big( \mathbf{H}_J^{\mathsf{T}} \mathbf{H}_J \big)^{-1} \mathbf{H}_J^{\mathsf{T}}$ and $\Delta\boldthetahat_{J,\bm{1}+\epsilon\bm{w}}=\boldthetahat_J(\bm{1}+\epsilon\bm{w})-\boldthetahat_J$, 
    the \emph{generalized influence function} for $\boldtheta_J$ and its parameter displacement are defined as
    \begin{align}
        \mathcal{I}\big(\bm{w}, \boldthetahat_J|\boldthetahat, \epsilon\big) 
        \mathrel{\mathop:}= 
        -\mathbf{H}_J^{+}
        \bm{g}_{\bm{w}},
        \quad
        \Delta\boldthetahat_{J,\bm{1}+\epsilon\bm{w}}
        :=
        -\epsilon\mathbf{H}_J^{+}\bm{g}_{\bm{w}}
        =
        \epsilon\mathcal{I}\big(\bm{w}, \boldthetahat_J \mid \boldthetahat, \epsilon\big).
        \label{eq:GIF}
    \end{align}
    \end{definition}
\end{tcolorbox}

The derivation of GIFs also begins with the first-order optimality condition \eqref{eq:first order optimality} as in the classical IFs.
As $\boldthetahat_{-J}$ are assumed to be stationary, the model changes from $\boldthetahat=[\boldthetahat_J, \boldthetahat_{-J}]^\mathsf{T}$ to $\boldthetahat(\epsilon,\bm{w})=[\boldthetahat_J(\epsilon, \bm{w}), \boldthetahat_{-J}]^\mathsf{T}$.
Then, a Taylor expansion of the first-order optimality condition in \eqref{eq:taylor_derivation_influence} gives the linearized full-model condition that the restricted update seeks to satisfy
\begin{align}
    0 &= \epsilon\bm{g}_{\bm{w}} +[\mathbf{H}_J|\mathbf{H}_{-J} ]
    \left(
    \begin{bmatrix} \boldthetahat_J(\bm{1}+\epsilon\bm{w}) \\ \boldthetahat_{-J} \end{bmatrix} -
    \begin{bmatrix} \boldthetahat_J \\ \boldthetahat_{-J} \end{bmatrix}
    \right)
    = \epsilon \bm{g}_{\bm{w}} + \mathbf{H}_J
    \Delta\boldthetahat_{J,\bm{1}+\epsilon\bm{w}}
    \label{eq:taylor_derivation_generalized_influence}
\end{align}
as the $\boldthetahat_{-J}$ terms cancel out.
Since $\mathbf{H}_J \in \mathbb{R}^{n\times k}$ and $\boldthetahat_J(\bm{1}+\epsilon\bm{w}) \in \mathbb{R}^{k}$ for $k<n$, \eqref{eq:taylor_derivation_generalized_influence} is an overdetermined linear problem and need not admit an exact solution.
We therefore define the restricted update as the minimum-residual problem \citep{ben2003-MoorePenroseInverse}
\begin{equation}
    \boldthetahat_J(\bm{1}+\epsilon\bm{w})
    :=
    \operatorname*{argmin}_{\boldtheta\in\mathbb{R}^{k}}
    \big\|
        \epsilon \bm{g}_{\bm{w}} + \mathbf{H}_J(\boldtheta-\boldthetahat_J)
    \big\|_2^2.
    \label{eq:l2_minimization_gif}
\end{equation}
Then, the minimum-residual
solution is given by
\begin{align}
    \Delta\boldthetahat_{J,\bm{1}+\epsilon\bm{w}}
    :=
    \boldthetahat_J(\bm{1}+\epsilon\bm{w})-\boldthetahat_J
    =
    -\epsilon\mathbf{H}_J^{+}\bm{g}_{\bm{w}}
    =
    \epsilon\mathcal{I}\big(\bm{w}, \boldthetahat_J \mid \boldthetahat, \epsilon\big).
    \label{eq:gif_linearized_solution}
\end{align}

\paragraph{The classical IFs as a special case.}
GIFs reduce to the classical influence update \eqref{eq:finite_linearized_update} when all parameters are editable, i.e., $\boldthetahat_J=\boldthetahat$.
Then, as $\mathbf{H}^{+}=\mathbf{H}^{-1}_{\bm{1}+\epsilon\bm{w}}$, we have
\begin{align}
    \mathcal{I}\big(\bm{w}, \boldthetahat | \boldthetahat, \epsilon \big) =
    -\mathbf{H}_{\bm{1}+\epsilon\bm{w}}^{-1} \bm{g}_{\bm{w}} =
    \mathcal{I}\big(\bm{w}, \boldthetahat | \epsilon\big).
    \label{eq:GIF_special_case}
\end{align}
Thus, GIFs are an extension of the classical IF formulation by choosing the parameter set $J$.


\subsection{The Model Edit Fidelity of Generalized Influence Functions}
\label{sec:The Accuracy of Generalized Influence Functions}

This subsection studies the approximation quality of GIFs under finite-step update.
Without loss of generality, we set $\epsilon=-1$ (i.e., unlearning $\bm{w}$), while the same argument extends to any value of $\epsilon$.
We define the model edited by the GIF as
\begin{align}
    \boldthetahat^{\mathsf{GIF}}(\bm{1}-\bm{w})
    :=
    \boldthetahat+\Delta
    ~\text{for}~
    \Delta =     \begin{bmatrix}
        \mathbf{H}_J^+\bm{g}_{\bm{w}}
        \mid
        \bm 0
    \end{bmatrix}^\mathsf{T}
\end{align}
which corresponds to the restricted linearized displacement
$-\epsilon\mathbf{H}_J^+\bm g_{\bm w}$ evaluated at $\epsilon=-1$.

We first analyze GIFs from the perspective of retraining fidelity.
In model editing, the finite influence update is useful only insofar as it
approximates the model that would have been obtained by retraining on the edited
empirical risk.
Thus, we measure how close the edited model
$\hat{\bm{\theta}}^{\mathsf{GIF}}(\bm 1-\bm w)$ is to the from-scratch retrained model
$\hat{\bm{\theta}}(\bm 1-\bm w)$ under an arbitrary evaluation functional.
Let $f$ be an evaluation functional on
$\Theta\subset\mathbb{R}^n$.
We define the approximation error as
\begin{align}
    \mathrm{Err}_{J}(\bm{w}\mid\hat{\bm{\theta}})
    :=
    f\big(\hat{\bm{\theta}}(\bm{1}-\bm{w})\big)
    -
    f\big(\hat{\bm{\theta}}^{\mathsf{GIF}}(\bm{1}-\bm{w})\big).
\end{align}
Typical choices of $f(\cdot)$ include the test loss,
$f(\theta)=\ell(x_{\mathrm{test}},y_{\mathrm{test}};\theta)$,
and the self-loss on the removed samples,
$f(\theta)=\sum_i w_i\ell(x_i,y_i;\theta)$.
The former measures the effect of the edit on generalization, while the latter
measures its effect on the data points selected by $\bm w$.

The following theorem shows that the GIF update admits a tighter approximation error bound than the full-parameter IF update
when the selected parameters are better conditioned and preserve a sufficiently large component of the target gradient.
\begin{tcolorbox}[
width=1.00\linewidth, colframe=box_margin, colback=box_background, boxrule=0.3mm, boxsep=-3pt, left=8pt, right=8pt]
    \begin{theorem}[Error bound of GIFs]
    \label{theorem:Error bound of the GIF}
        Let the functional $f(\cdot)$ and the Hessian map $\nabla^2\mathcal{L}_{\bm{1}-\bm{w}}(\cdot)$ be $C_f$- and $C_\mathbf{H}$-Lipschitz continuous on a convex neighborhood $\mathcal U$ containing $\boldthetahat$, $\boldthetahat^{\mathsf{GIF}}(\bm{1}-\bm{w})$, and the retrained optimum $\boldthetahat(\bm{1}-\bm{w})$.
        Assume $\nabla\mathcal L_{\bm 1}(\boldthetahat)=\bm 0$.
        If $\mathcal{L}_{\bm{1}-\bm{w}}$ is $\sigma_{\min}(\mathbf H_{\bm 1-\bm w})$-strongly convex on $\mathcal U$, with $\sigma_{\min}(\mathbf H_{\bm 1-\bm w})>0$,
        then, for $\bm\Pi_J=\mathbf H_J\mathbf H_J^+$ and $\rho(J)=\|(\mathbf I-\bm\Pi_J)\bm g_{\bm w}\|_2$, the evaluation error is upper-bounded as
        \begin{equation}
            \resizebox{\dimexpr\linewidth-4.5em\relax}{!}{$\displaystyle
            \big|\mathrm{Err}_{J}(\bm{w}\mid\boldthetahat)\big|
            \le
            \frac{C_f}{\sigma_{\min}(\mathbf H_{\bm 1-\bm w})}
            \left[
                \frac{C_\mathbf H}{2}\|\mathbf H_J^+\bm g_{\bm w}\|_2^2
                +\rho(J)
            \right]
            \le
            \frac{C_f}{\sigma_{\min}(\mathbf H_{\bm 1-\bm w})}
            \left[
                \frac{C_\mathbf H}{2}
                \frac{\|\bm g_{\bm w}\|_2^2}{\sigma_{\min}(\mathbf H_J)^2}
                +\rho(J)
            \right].$}
            \label{eq:error bound}
        \end{equation}
    \end{theorem}
\end{tcolorbox}

The bound in \textbf{Theorem~\ref{theorem:Error bound of the GIF}} is governed by
$\sigma_{\min}(\mathbf H_J)$, the smallest singular value of the Hessian block
restricted to the selected parameter subset $J$, and the projection residual $\rho(J)$.
The residual measures the component of the target gradient that cannot be represented by the selected Hessian columns.
Here the reference-point quantity $\sigma_{\min}(\mathbf H_{\bm 1-\bm w})$ is assumed to lower-bound the Hessian eigenvalues throughout $\mathcal U$, so the guarantee applies locally to the reference, edited, and retrained models in this neighborhood.
$\mathbf H_J^\top\mathbf H_J$ is a principal submatrix of the full Gram
matrix $\mathbf H_{\bm 1-\bm w}^\top\mathbf H_{\bm 1-\bm w}$, up to a coordinate
reordering. 
By the Cauchy interlacing theorem \cite{Horn2012-matrixanalysis},
\begin{align}
    \sigma_{\min}^2(\mathbf H_J)
    =\lambda_{\min}(\mathbf H_J^\top\mathbf H_J)
    \ge
    \lambda_{\min}(\mathbf H_{\bm 1-\bm w}^\top\mathbf H_{\bm 1-\bm w})
    =\sigma_{\min}^2(\mathbf H_{\bm 1-\bm w}),
\end{align}
and hence $\sigma_{\min}(\mathbf H_J)\ge
\sigma_{\min}(\mathbf H_{\bm 1-\bm w})$.
This inequality controls the curvature-dependent term in \eqref{eq:error bound}, while a small $\rho(J)$ preserves the target-gradient component needed for the edit.

When all parameters are editable, $\bm\Pi_J=\mathbf I$ and $\mathbf H_J^+=\mathbf H_{\bm 1-\bm w}^{-1}$.
Provided that the full-parameter edited model also lies in $\mathcal U$, the corresponding bound is
\begin{align}
    \big|\mathrm{Err}_{\mathrm{IF}}(\bm w\mid\boldthetahat)\big|
    \le
    \frac{C_f}{\sigma_{\min}(\mathbf H_{\bm 1-\bm w})}
    \frac{C_\mathbf H}{2}
    \|\mathbf H_{\bm 1-\bm w}^{-1}\bm g_{\bm w}\|_2^2.
    \label{eq:full_if_error_bound}
\end{align}
Here, $\mathrm{Err}_{\mathrm{IF}}$ denotes the evaluation error when all parameters are editable.
Since $C_f/\sigma_{\min}(\mathbf H_{\bm 1-\bm w})$ is common to both bounds, the GIF bound is no larger whenever
\begin{align}
    \frac{C_\mathbf H}{2}\|\mathbf H_J^+\bm g_{\bm w}\|_2^2+\rho(J)
    \le
    \frac{C_\mathbf H}{2}\|\mathbf H_{\bm 1-\bm w}^{-1}\bm g_{\bm w}\|_2^2.
    \label{eq:gif_if_bound_comparison}
\end{align}
Consequently, GIF updates admit a tighter error bound when the reduction in the curvature-dependent term exceeds the projection residual.

We next provide a complementary analysis focused on the effect of the GIF update on the edited data themselves.
While Theorem~\ref{theorem:Error bound of the GIF} bounds the discrepancy between the retrained model and the edited model under an arbitrary evaluation functional $f$, here we ask whether the GIF update moves the edited data in the desired direction.

Define the masked empirical risk
$    \mathcal L_{\bm w}(\bm\theta)
    :=
    \frac{1}{m}\sum_{i=1}^m w_i\ell(z_i,\bm\theta)$.
We define the target-risk response induced by the GIF update as
\begin{align}
    \Delta \mathcal L_J(\bm w\mid\hat{\bm\theta})
    :=
    \mathcal L_{\bm w}
    \big(
        \hat{\bm\theta}^{\mathsf{GIF}}(\bm 1-\bm w)
    \big)
    -
    \mathcal L_{\bm w}(\hat{\bm\theta}).
\end{align}
$\Delta \mathcal{L}_J(\bm w\,|\,\boldthetahat)$ compares the loss on the
selected data before and after applying the GIF-based unlearning update.
A positive value means that the selected data become harder for the edited model
to fit, which is the desired direction in unlearning.
\begin{tcolorbox}[
width=1.00\linewidth, colframe=box_margin, colback=box_background, boxrule=0.3mm, boxsep=-3pt, left=8pt, right=8pt]
    \begin{theorem}[Response of the self-loss]
    \label{theorem:Self-loss bound of GIFs}
    Suppose that $\mathcal L_{\bm w}$ is locally convex and let $\bm{\Pi}_J=\mathbf{H}_J\mathbf{H}_J^+$.
    If $\bm g_{\bm w}$ and $\Delta$ are aligned, i.e.,
    $
        \bm g_{\bm w}^{\mathsf T}\Delta\ge0,
    $
    then, for some $c\in[0,1]$, we have
    \begin{align}
        \Delta \mathcal L_J(\bm w\mid\boldthetahat)
        \ge
        \left(
            1 - c
            +
            \frac{\sigma_{\min}(\mathbf{H} _{\bm w})}
            {2\sigma_{\max}(\mathbf H_J)^2}
        \right)
        \|\bm \Pi_J\bm g_{\bm w}\|_2^2.
        \label{eq:theorem2 self-loss response}
    \end{align}
    \end{theorem}
\end{tcolorbox}
\textbf{Theorem~\ref{theorem:Self-loss bound of GIFs}} shows that effective forgetting
requires a large projection response $\|\bm\Pi_J\bm g_{\bm w}\|_2^2$.
Since $\bm\Pi_J=\mathbf H_J\mathbf H_J^+$ projects the target gradient onto the
Hessian-column space induced by the selected parameters, this term measures how
much of the target-gradient direction can be realized by updating only $J$.
Thus, parameter selection should choose a subset that both avoids unstable
curvature directions and captures a large target-gradient response.

\paragraph{Why full-parameter updates fail?}
\textbf{Theorem~\ref{theorem:Error bound of the GIF}} shows that the curvature-dependent term in the error bound grows as the minimum singular value decreases.
In deep neural networks, the full parameter space typically contains many flat directions with near-zero singular values.
These directions can amplify small gradient perturbations into unstable updates, enlarging the full-parameter bound in \eqref{eq:full_if_error_bound}.
Restricting the update to a well-conditioned parameter subset $J$ can avoid these flat directions and yield a tighter finite-step error bound when the reduction in the curvature-dependent term exceeds the projection residual, as in \eqref{eq:gif_if_bound_comparison}.

Moreover, \textbf{Theorem~\ref{theorem:Self-loss bound of GIFs}} provides another explanation for why full-parameter updates can be suboptimal.
When GIFs are extended to the full-parameter IF case, 
Eq.~\eqref{eq:theorem2 self-loss response} reduces to
\begin{align}
    \Delta \mathcal L_{\mathrm{IF}}(\bm w\mid\boldthetahat)
    \ge
    \left(
        1
        +
        \frac{\sigma_{\min}(\mathbf{H}_{\bm w})}
        {2\sigma_{\max}(\mathbf H_{\bm 1-\bm w})^2}
    \right)
    \|\bm g_{\bm w}\|_2^2 .
\end{align}
However, in well-trained models, $\bm g_{\bm w}$ could be sparse, concentrated on a relatively small subset of parameter directions.
In such cases, matching the full-parameter response is sufficient for $J$ to cover the high-magnitude directions that carry most of the gradient \cite{lin2018deep, sun17c-meProp}.
Consequently, one may have
$\|\bm\Pi_J\bm g_{\bm w}\|_2^2\approx \|\bm g_{\bm w}\|_2^2$ even when
$J$ contains only a small fraction of the parameters.
When $J$ preserves most of the target-gradient response, the comparison is
mainly governed by the curvature denominator.
Since Cauchy interlacing gives
$
    \sigma_{\max}(\mathbf H_J)
    \le
    \sigma_{\max}(\mathbf H_{\bm 1-\bm w}),
$
the full-parameter update can increase this denominator without a commensurate
gain in self-loss response.
Thus, a well-chosen GIF update can yield a stronger effective edit than the
classical IF update.

\subsection{Identifying Suitable Target Parameters}
\label{sec:param_sel}

Accurate GIF updates require an editable subset $\boldtheta_J$ that captures the desirable direction while avoiding unnecessary changes to parameters that support retained behavior.
\textbf{Theorems~\ref{theorem:Error bound of the GIF}} and \textbf{\ref{theorem:Self-loss bound of GIFs}} show that the parameter subspace should avoid flat curvature directions, while preserving the target-gradient response.
Therefore, we design scoring rules based on these two requirements.

Since the curvature often has strong interactions among nearby parameters in the computational graph, we use a blockwise-sparse curvature approximation and evaluate parameters blockwise \cite{grosse2023-ekfac}.
Let $B\subset\{1,\ldots,n\}$ be a candidate parameter block, and score how suitable each block is for the editable subset.
The most direct criterion is curvature-aware parameter selection (CAPS), defined by
\begin{align}
    S(B)
    :=
    \bm g_{\bm w,B}^{\mathsf T}
    (\mathbf C_B+\lambda \mathbf I)^{-1}
    \bm g_{\bm w,B},
\end{align}
where $\bm g_{\bm w,B}$ is the gradient restricted to $B$ and $\mathbf C_B$ is a local curvature proxy.
Motivated by the min-residual structure of GIFs \eqref{eq:l2_minimization_gif}, CAPS favors blocks with large target-gradient while penalizing low local curvature.
Appendix~\ref{appendix:blockwise_caps} derives why CAPS suggests a well-conditioned parameter subspace.

We also consider two lighter selection criteria, Top-$\rho\%$ gradients and
Top-$\rho\%$ outputs.
For these criteria, each block consists of the parameters that produce one
coordinate or channel of a layer output tensor.
Top-$\rho\%$ gradients assigns a score to each block using the magnitude of the
backward-pass gradient at the corresponding output unit on the target data.
It then selects the highest-scoring blocks until they cover $\rho\%$ of the
layer's trainable parameters.
Top-$\rho\%$ outputs uses the same block construction, but scores each block using
the magnitude of the forward-pass activation at the corresponding output unit on
the target data.
Thus, Top-$\rho\%$ gradients follows the target loss more directly, whereas
Top-$\rho\%$ outputs provides a cheaper forward-pass proxy.
These criteria are related to pruning saliency methods for selecting relevant
parameters \cite{das2024-pruning, sung2024-pruning, Shi19-pruning}, but here they define the editable subspace rather than compressing the network.

\section{Pseudo-LiSSA for Inverse Sub-Hessian-Vector Product}
\label{sec:Numerical Approximation of GIF}

We introduce p-LiSSA, which converges to the inverse sub-Hessian vector product, i.e., $\mathbf{H}_J^{+}\mathbf{g}$.
For any $\mu>0$, we define p-LiSSA by the following iteration,
\begin{align}
    \mathcal{I}_{k}
    =
    \mathcal{I}_0
    +
    (\mathbf{I} - \mu\mathbf{H}_J^\mathsf{T}\mathbf{H}_J)\mathcal{I}_{k-1},
    \quad
    \mathcal{I}_0 = \mu\mathbf{H}_J^\mathsf{T}\bm{g}_{\bm{w}}.
    \label{eq:approximate_series_GIF}
\end{align}
If $\rho(\mathbf{I}-\mu\mathbf{H}_J^\mathsf{T}\mathbf{H}_J)<1$, 
$\mathcal{I}_{k}$ converges to
$\mathbf{H}^{+}_{J}\bm{g}_{\bm{w}}
\mathrel{\mathop:}= \mathcal{I}\big(\bm{w}, \boldthetahat_J|\boldthetahat, \epsilon\big)$.
The theorem below shows that p-LiSSA converges to iHVP without auxiliary regularization.
This contributes to reducing the error of IFs by facilitating more accurate iHVP computation.
\begin{tcolorbox}[
width=1.00\linewidth, colframe=box_margin, colback=box_background, boxrule=0.3mm, boxsep=-3pt, left=8pt, right=8pt]
\begin{theorem}[Convergence of p-LiSSA]
\label{theorem:Convergence of p-LiSSA}
    Under the above definition, \\
    \textbf{\emph{i)}} p-LiSSA is scale-invariant, i.e., scalar multiplication on $\ell(\cdot)$ does not change $\mathcal{I}\big(\bm{w}, \boldthetahat_J|\boldthetahat, \epsilon\big)$. \\
    \textbf{\emph{ii)}} If $0<\mu<2/\lambda_{\max}(\mathbf{H}_J^\mathsf{T}\mathbf{H}_J)$, then p-LiSSA converges to $\mathbf{H}_J^{+}\bm{g}_{\bm{w}}$.
    \\
    \textbf{\emph{iii)}} If $\mu = \frac{2}{\lambda_{\max}(\mathbf{H}_J^\mathsf{T}\mathbf{H}_J) + \lambda_{\min}(\mathbf{H}_J^\mathsf{T}\mathbf{H}_J)}$, the convergence rate of p-LiSSA is given by $\frac{\kappa-1}{\kappa+1}$, where $\kappa = \frac{\lambda_{\max}(\mathbf{H}_J^\top\mathbf{H}_J)}{\lambda_{\min}(\mathbf{H}_J^\top\mathbf{H}_J)}$.
\end{theorem}
\end{tcolorbox}

\textbf{Theorem~\ref{theorem:Convergence of p-LiSSA}} holds as $\mathbf{H}_J^\mathsf{T}\mathbf{H}_J$ is positive definite under the definition of p-LiSSA, and the eigenvalues of $\mathbf{I}-\mu\mathbf{H}_J^\mathsf{T}\mathbf{H}_J$ are $1-\mu\lambda_i$; hence $\rho(\mathbf{I}-\mu\mathbf{H}_J^\mathsf{T}\mathbf{H}_J)<1$ for $0<\mu<2/\lambda_{\max}(\mathbf{H}_J^\mathsf{T}\mathbf{H}_J)$.
In practice, we can adaptively decrease $\mu$ whenever p-LiSSA diverges rather than explicitly computing $\sigma_{\max}(\mathbf{H}_J^\mathsf{T}\mathbf{H}_J)$, which is a standard damping strategy in numerical optimization, such as the Levenberg-Marquardt algorithm \cite{Transtrum11-Levenberg} and classical backtracking line search \cite{Dimitri16-nonlinear_programming}.
Furthermore, we show that the convergence rate is governed by the condition number $\kappa = \lambda_{\max}/\lambda_{\min}$ of $\mathbf{H}_J^\top\mathbf{H}_J$. Since low-curvature directions correspond to small eigenvalues of $\mathbf{H}_J^\top\mathbf{H}_J$, including them drives $\lambda_{\min} \to 0$, causing $\big(({\kappa-1})/({\kappa+1})\big)^k \to 1$ for any fixed $k$.

\paragraph{Numerical implementation.} Note that $\mathbf{H}_J^\mathsf{T} \mathbf{H}_J \mathcal{I}_{k-1}$ in p-LiSSA cannot be directly computed by a single HVP.
Instead, we can obtain $\mathbf{H}_J^\mathsf{T} \mathbf{H}_J \mathcal{I}_{k-1}$ by applying HVP twice and selecting the $J$ entries from the result vector, where 
\begin{align}
    \mathbf{H}
    \cunderbrace{
        \mathbf{H}
        \begin{bmatrix}
        \mathcal{I}_{k-1} \vphantom{\mathbf{H}_J^\mathsf{T}\mathbf{H}_J \mathcal{I}_{k-1}} \\[1mm]
        \mathbf{0} \vphantom{\mathbf{H}_{-J}^\mathsf{T}\mathbf{H}_J \mathcal{I}_{k-1}}
        \end{bmatrix}    }
    {\text{First HVP}}
    =
    \cunderbrace{
    \begin{bmatrix}
    \mathbf{H}_J^\mathsf{T} \\[1mm]
    \mathbf{H}_{-J}^\mathsf{T}
    \end{bmatrix}
    \mathbf{H}_J \mathcal{I}_{k-1}
    }
    {\text{Second HVP}}
    =
    \cunderbrace{ 
    \begin{bmatrix}
    \mathbf{H}_J^\mathsf{T}\mathbf{H}_J \mathcal{I}_{k-1} \\[1mm]
    \textcolor{lightgray}{\cancel{\mathbf{H}_{-J}^\mathsf{T}\mathbf{H}_J \mathcal{I}_{k-1}}}
    \end{bmatrix}.
    }{\text{Dropping $\boldtheta_{-J}$ entries}}
    \label{eq:1st and 2nd HVPs}
\end{align}
While LiSSA provides an unbiased estimate using single mini-batches per iteration \cite{Agarwal2017-HessianEstimator},
applying the same batch to both the first and the second HVP in \eqref{eq:1st and 2nd HVPs} leads to a biased estimator, as
$\mathbb{E}[\mathbf{H}_{J,B}^\mathsf{T}\mathbf{H}_{J,B}]
\neq
\mathbb{E}[\mathbf{H}_{J,B}]^\mathsf{T}\mathbb{E}[\mathbf{H}_{J,B}]$.
Thus, we adopt two independent mini-batches per iteration,
which the following theorem shows is sufficient to guarantee
that p-LiSSA is an unbiased estimator of GIFs.
\begin{tcolorbox}[
width=1.00\linewidth, colframe=box_margin, colback=box_background, boxrule=0.3mm, boxsep=-3pt, left=8pt, right=8pt]
\begin{theorem}[Unbiasedness of p-LiSSA]
    \label{theorem:Unbiasedness of p-LiSSA}
    Assume that the batches $B_{k,1}, B_{k,2}$ are drawn independently to compute $\mathbf{H}_{J, {k,1}}$ and $\mathbf{H}_{J, {k,2}}$.
    Define the stochastic recurrence as
    \begin{align}
        \tilde{\mathcal{I}}_k = \tilde{\mathcal{I}}_0\! + \!\left(\mathbf{I} - \mathbf{H}_{J, {k,1}}^{\mathsf{T}}\mathbf{H}_{J, {k,2}}\right)\!\tilde{\mathcal{I}}_{k-1},~
        \tilde{\mathcal{I}}_0 = \mathbf{H}_{J, {0,1}}^{\mathsf{T}} \bm{g}_{\bm{w}}
        .
        \nonumber
    \end{align}
    Then, for any $k \ge 0$, the expected value of $\tilde{\mathcal{I}}_k$ satisfies
    $
        \textstyle
        \mathbb{E}[\tilde{\mathcal{I}}_k] = \sum_{j=0}^{k} \left(\mathbf{I} - \mathbf{H}_J^{\mathsf{T}}\mathbf{H}_J\right)^j \mathbf{H}_J^{\mathsf{T}}\bm{g}_{\bm{w}}.
    $
    Consequently, as $k \to \infty$, $\mathbb{E}[\tilde{\mathcal{I}}_k]$ converges to $\mathbf{H}_J^{+}\bm{g}_{\bm{w}}$.
\end{theorem}
\end{tcolorbox}
Taken together, Alg.~\ref{alg:GIF} provides the complete computational procedure for GIFs.

\Needspace{0.42\textheight}
\section{Numerical Experiments}
\label{sec:Numerical Experiments}

\begin{wrapfigure}{r}{0.43\textwidth}
\small
\refstepcounter{algorithm}
\noindent\rule{\linewidth}{0.4pt}

\noindent\textbf{Algorithm~\thealgorithm. Generalized Influence Functions}
\label{alg:GIF}
\hfill
\vspace{-5.5mm}
\noindent\rule{\linewidth}{0.4pt}

\begin{algorithmic}[1]
\STATE {\bfseries Input:} $\boldthetahat$, $\bm w\in\{0,1\}^m$, $\mu$.
\STATE Determine $J$ via the selection methods.
\STATE $\mathcal I_0 \leftarrow \mu\,\mathbf H_J^\mathsf{T}\bm g_{\bm w}$, $k \leftarrow 0$.
\REPEAT
    \STATE \textit{\# Using two mini-batches}
    \STATE $\mathcal I_{k+1} \leftarrow \mu\,\mathbf H_J^\mathsf{T}\bm g_{\bm w}
    + (\mathbf I-\mu\,\mathbf H_J^\mathsf{T}\mathbf H_J)\mathcal I_k$.
    \IF{$\mathcal I_{k+1}$ is diverging}
        \STATE $\mu \leftarrow \mu/2$, $\mathcal I_0 \leftarrow \mu\,\mathbf H_J^\mathsf{T}\bm g_{\bm w}$.
        \STATE Restart the loop with $k \leftarrow 0$.
    \ENDIF
    \STATE $k \leftarrow k+1$.
\UNTIL{convergence}
\end{algorithmic}
\vspace{-2mm}
\noindent\rule{\linewidth}{0.4pt}
\vspace{-6mm}
\end{wrapfigure}

We organize our experiments around the three main claims of this paper.
First, we evaluate whether GIFs improve the
model-editing trade-off in class-level unlearning, where the edited model should
forget the target data while preserving performance on the retained data.
Second, we study the role of parameter selection by varying the editable parameter subset while keeping the influence update rule fixed.
Third, we test whether p-LiSSA provides a more stable numerical approximation to the iHVP than existing iHVP solvers.
Together, these experiments test our claim that classical IFs become inaccurate as full-parameter updates move irrelevant or low-curvature directions, and as the required iHVP computation is numerically unstable.
Implementation and reproducibility details are provided in Appendix~\ref{appendix:reproducibility}.

\subsection{GIFs Improve Model Edits}

\begin{table*}[htb]
\centering
\caption{Data-removal results for full-parameter influence-function, restricted, and trajectory-aware baselines.
\textit{Retain} is the accuracy on non-target examples and should be high;
\textit{Unlearn} is the accuracy on target examples and should be low.}
\label{tab:model_edit_comparison}
\setlength{\tabcolsep}{4pt}
\begin{adjustbox}{width=\textwidth}
\begin{tabular}{l cc cc cc cc cc}
\toprule
\multirow{2}{*}{Method}
& \multicolumn{2}{c}{\shortstack{MNIST\\{\scriptsize FCN}}}
& \multicolumn{2}{c}{\shortstack{CIFAR-10\\{\scriptsize ResNet-18}}}
& \multicolumn{2}{c}{\shortstack{SVHN\\{\scriptsize VGG-11}}}
& \multicolumn{2}{c}{\shortstack{Newsgroups\\{\scriptsize BERT-L2-H128}}}
& \multicolumn{2}{c}{\shortstack{PubMed 20k RCT\\{\scriptsize BERT-L2-H128}}} \\
\cmidrule(lr){2-3}
\cmidrule(lr){4-5}
\cmidrule(lr){6-7}
\cmidrule(lr){8-9}
\cmidrule(lr){10-11}
& Retain ($\uparrow$) & Unlearn ($\downarrow$)
& Retain ($\uparrow$) & Unlearn ($\downarrow$)
& Retain ($\uparrow$) & Unlearn ($\downarrow$)
& Retain ($\uparrow$) & Unlearn ($\downarrow$)
& Retain ($\uparrow$) & Unlearn ($\downarrow$) \\
\midrule

\arrayrulecolor{lightgray}
Retrain          & 98.09 {\scriptsize \textcolor{gray}{$\pm$ 0.00}} & 0.00 {\scriptsize \textcolor{gray}{$\pm$ 0.00}} & 85.86 {\scriptsize \textcolor{gray}{$\pm$ 0.00}} & 0.00 {\scriptsize \textcolor{gray}{$\pm$ 0.00}} & 95.51 {\scriptsize \textcolor{gray}{$\pm$ 0.00}} & 0.00 {\scriptsize \textcolor{gray}{$\pm$ 0.00}} & 65.52 {\scriptsize \textcolor{gray}{$\pm$ 0.00}} & 0.00 {\scriptsize \textcolor{gray}{$\pm$ 0.00}} & 87.64 {\scriptsize \textcolor{gray}{$\pm$ 0.00}} & 0.00 {\scriptsize \textcolor{gray}{$\pm$ 0.00}} \\
\arrayrulecolor{black}
\cmidrule(l{2pt}r{2pt}){1-11}
\arrayrulecolor{lightgray}

Classic IF~\cite{Koh2017_IF} & 93.38 {\scriptsize \textcolor{gray}{$\pm$ 2.02}} & 1.33 {\scriptsize \textcolor{gray}{$\pm$ 1.47}} & 84.50 {\scriptsize \textcolor{gray}{$\pm$ 0.11}} & 0.58 {\scriptsize \textcolor{gray}{$\pm$ 0.14}} & 93.94 {\scriptsize \textcolor{gray}{$\pm$ 0.17}} & 0.11 {\scriptsize \textcolor{gray}{$\pm$ 0.12}} & 60.06 {\scriptsize \textcolor{gray}{$\pm$ 0.50}} & 5.83 {\scriptsize \textcolor{gray}{$\pm$ 0.61}} & 86.57 {\scriptsize \textcolor{gray}{$\pm$ 0.09}} & 2.51 {\scriptsize \textcolor{gray}{$\pm$ 0.04}} \\
\cmidrule(l{2pt}r{2pt}){1-1}\cmidrule(l{2pt}r{2pt}){2-11}

Freezing & 86.79 {\scriptsize \textcolor{gray}{$\pm$ 7.64}} & 4.04 {\scriptsize \textcolor{gray}{$\pm$ 4.84}} & 79.99 {\scriptsize \textcolor{gray}{$\pm$ 11.90}} & 1.35 {\scriptsize \textcolor{gray}{$\pm$ 2.64}} & 93.76 {\scriptsize \textcolor{gray}{$\pm$ 0.00}} & 0.00 {\scriptsize \textcolor{gray}{$\pm$ 0.00}} & 62.50 {\scriptsize \textcolor{gray}{$\pm$ 0.02}} & 14.73 {\scriptsize \textcolor{gray}{$\pm$ 0.00}} & 83.65 {\scriptsize \textcolor{gray}{$\pm$ 0.30}} & 3.67 {\scriptsize \textcolor{gray}{$\pm$ 0.23}} \\
\cmidrule(l{2pt}r{2pt}){1-1}\cmidrule(l{2pt}r{2pt}){2-11}

Second-order IF~\cite{Basu20-SecondIF} & 93.66 {\scriptsize \textcolor{gray}{$\pm$ 2.55}} & 10.90 {\scriptsize \textcolor{gray}{$\pm$ 30.97}} & 84.54 {\scriptsize \textcolor{gray}{$\pm$ 0.09}} & 0.60 {\scriptsize \textcolor{gray}{$\pm$ 0.13}} & 93.92 {\scriptsize \textcolor{gray}{$\pm$ 0.16}} & 0.15 {\scriptsize \textcolor{gray}{$\pm$ 0.16}} & 60.04 {\scriptsize \textcolor{gray}{$\pm$ 0.58}} & 5.92 {\scriptsize \textcolor{gray}{$\pm$ 0.79}} & 86.70 {\scriptsize \textcolor{gray}{$\pm$ 0.20}} & 2.53 {\scriptsize \textcolor{gray}{$\pm$ 0.05}} \\
\cmidrule(l{2pt}r{2pt}){1-1}\cmidrule(l{2pt}r{2pt}){2-11}

TracIn~\cite{Pruthi2020-TracIn} & 94.48 {\scriptsize \textcolor{gray}{$\pm$ 0.65}} & 0.68 {\scriptsize \textcolor{gray}{$\pm$ 0.61}} & 83.08 {\scriptsize \textcolor{gray}{$\pm$ 0.00}} & 1.00 {\scriptsize \textcolor{gray}{$\pm$ 0.00}} & 94.36 {\scriptsize \textcolor{gray}{$\pm$ 0.00}} & 0.00 {\scriptsize \textcolor{gray}{$\pm$ 0.00}} & 60.65 {\scriptsize \textcolor{gray}{$\pm$ 0.00}} & 11.60 {\scriptsize \textcolor{gray}{$\pm$ 0.00}} & 83.81 {\scriptsize \textcolor{gray}{$\pm$ 0.46}} & 2.19 {\scriptsize \textcolor{gray}{$\pm$ 0.12}} \\
\cmidrule(l{2pt}r{2pt}){1-1}\cmidrule(l{2pt}r{2pt}){2-11}

Diff-In~\cite{tan2025diffin} & 81.03 {\scriptsize \textcolor{gray}{$\pm$ 6.90}} & 4.84 {\scriptsize \textcolor{gray}{$\pm$ 2.88}} & 77.86 {\scriptsize \textcolor{gray}{$\pm$ 3.18}} & 4.00 {\scriptsize \textcolor{gray}{$\pm$ 2.96}} & 89.79 {\scriptsize \textcolor{gray}{$\pm$ 1.15}} & 1.08 {\scriptsize \textcolor{gray}{$\pm$ 0.80}} & 55.53 {\scriptsize \textcolor{gray}{$\pm$ 0.38}} & 9.89 {\scriptsize \textcolor{gray}{$\pm$ 0.51}} & 76.73 {\scriptsize \textcolor{gray}{$\pm$ 0.99}} & 1.41 {\scriptsize \textcolor{gray}{$\pm$ 0.16}} \\
\cmidrule(l{2pt}r{2pt}){1-1}\cmidrule(l{2pt}r{2pt}){2-11}

HyperInf~\cite{zhou2025-hyperinf} & 92.53 {\scriptsize \textcolor{gray}{$\pm$ 2.92}} & 1.81 {\scriptsize \textcolor{gray}{$\pm$ 1.44}} & 83.44 {\scriptsize \textcolor{gray}{$\pm$ 0.19}} & 0.57 {\scriptsize \textcolor{gray}{$\pm$ 0.50}} & 94.29 {\scriptsize \textcolor{gray}{$\pm$ 0.12}} & 0.06 {\scriptsize \textcolor{gray}{$\pm$ 0.11}} & 59.52 {\scriptsize \textcolor{gray}{$\pm$ 0.36}} & 8.06 {\scriptsize \textcolor{gray}{$\pm$ 0.47}} & 86.34 {\scriptsize \textcolor{gray}{$\pm$ 0.27}} & 3.43 {\scriptsize \textcolor{gray}{$\pm$ 0.31}} \\
\cmidrule(l{2pt}r{2pt}){1-1}\cmidrule(l{2pt}r{2pt}){2-11}

DataInf~\cite{kwon2024datainf} & 91.12 {\scriptsize \textcolor{gray}{$\pm$ 2.17}} & 2.62 {\scriptsize \textcolor{gray}{$\pm$ 1.17}} & 85.31 {\scriptsize \textcolor{gray}{$\pm$ 0.01}} & 0.30 {\scriptsize \textcolor{gray}{$\pm$ 0.00}} & 94.68 {\scriptsize \textcolor{gray}{$\pm$ 0.00}} & 0.00 {\scriptsize \textcolor{gray}{$\pm$ 0.00}} & 63.33 {\scriptsize \textcolor{gray}{$\pm$ 0.00}} & 3.76 {\scriptsize \textcolor{gray}{$\pm$ 0.00}} & 85.15 {\scriptsize \textcolor{gray}{$\pm$ 0.39}} & 1.69 {\scriptsize \textcolor{gray}{$\pm$ 0.25}} \\
\cmidrule(l{2pt}r{2pt}){1-1}\cmidrule(l{2pt}r{2pt}){2-11}

Schioppa~\cite{schioppa2022-scalingIF} & 94.82 {\scriptsize \textcolor{gray}{$\pm$ 2.04}} & 0.56 {\scriptsize \textcolor{gray}{$\pm$ 0.62}} & 60.27 {\scriptsize \textcolor{gray}{$\pm$ 0.56}} & 11.01 {\scriptsize \textcolor{gray}{$\pm$ 0.70}} & 93.13 {\scriptsize \textcolor{gray}{$\pm$ 0.00}} & 0.00 {\scriptsize \textcolor{gray}{$\pm$ 0.00}} & 63.05 {\scriptsize \textcolor{gray}{$\pm$ 1.30}} & 2.73 {\scriptsize \textcolor{gray}{$\pm$ 1.65}} & 85.40 {\scriptsize \textcolor{gray}{$\pm$ 1.24}} & 2.98 {\scriptsize \textcolor{gray}{$\pm$ 0.71}} \\
\cmidrule(l{2pt}r{2pt}){1-1}\cmidrule(l{2pt}r{2pt}){2-11}

EKFAC~\cite{grosse2023-ekfac,martens15-kfac} & 91.45 {\scriptsize \textcolor{gray}{$\pm$ 2.53}} & 1.30 {\scriptsize \textcolor{gray}{$\pm$ 0.83}} & 83.62 {\scriptsize \textcolor{gray}{$\pm$ 0.00}} & 0.50 {\scriptsize \textcolor{gray}{$\pm$ 0.00}} & 94.62 {\scriptsize \textcolor{gray}{$\pm$ 0.00}} & 0.00 {\scriptsize \textcolor{gray}{$\pm$ 0.00}} & 61.69 {\scriptsize \textcolor{gray}{$\pm$ 0.00}} & 4.39 {\scriptsize \textcolor{gray}{$\pm$ 0.00}} & 85.55 {\scriptsize \textcolor{gray}{$\pm$ 0.26}} & 1.614 {\scriptsize \textcolor{gray}{$\pm$ 0.06}} \\
\arrayrulecolor{black}
\cmidrule[0.75pt](l{2pt}r{2pt}){1-11}
\arrayrulecolor{lightgray}
\textbf{GIF (Ours)}    & \textbf{94.95 {\scriptsize \textcolor{gray}{$\pm$ 2.33}}} & \textbf{0.42 {\scriptsize \textcolor{gray}{$\pm$ 0.38}}} & \textbf{85.87 {\scriptsize \textcolor{gray}{$\pm$ 0.04}}} & \textbf{0.10 {\scriptsize \textcolor{gray}{$\pm$ 0.00}}} & \textbf{95.14 {\scriptsize \textcolor{gray}{$\pm$ 0.03}}} & \textbf{0.00 {\scriptsize \textcolor{gray}{$\pm$ 0.00}}} & \textbf{65.45 {\scriptsize \textcolor{gray}{$\pm$ 0.05}}} & \textbf{0.00 {\scriptsize \textcolor{gray}{$\pm$ 0.00}}} & \textbf{87.20 {\scriptsize \textcolor{gray}{$\pm$ 0.21}}} & \textbf{1.610 {\scriptsize \textcolor{gray}{$\pm$ 0.31}}} \\

\arrayrulecolor{black}
\bottomrule
\end{tabular}
\end{adjustbox}
\end{table*}

We first show that GIFs can produce better model edits in class-level data removal tasks than the IF baselines across multiple datasets and model architectures.
In the class-removal task, each method edits a trained model to remove the behavior associated with a target class.
The edited model is evaluated on two disjoint metrics.
\textit{Retain} is the accuracy on non-target examples and measures whether
unrelated behavior is preserved.
\textit{Unlearn} is the accuracy on target examples and measures how much
target behavior remains after editing.
We adopt the model trained from scratch without the target class as the reference, 
since IF and GIF are derived to approximate the model change induced by retraining under the corresponding edited objective.

We compare GIFs using a 10\% target-relevant parameter subset against full-parameter IF baselines, namely
classical IF~\cite{Koh2017_IF} and second-order IF~\cite{Basu20-SecondIF}, and
with recent influence-estimation schemes including
TracIn~\cite{Pruthi2020-TracIn}, HyperInf~\cite{zhou2025-hyperinf},
DataInf~\cite{kwon2024datainf}, Schioppa~\cite{schioppa2022-scalingIF}, and
EKFAC~\cite{grosse2023-ekfac,martens15-kfac}. We also include Freezing\footnote{Appendix~\ref{appendix:gif_freezing} explains how the Freezing update differs from GIF's full-gradient residual update.}, which solves the restricted ERM update in \eqref{eq:freezing_update}, and the trajectory-aware Diff-In estimator~\cite{tan2025diffin}.

Table~\ref{tab:model_edit_comparison} shows that GIF achieves the highest retain accuracy among the editing methods on all five benchmarks, while reducing unlearn accuracy to 0.00--1.61\%. Its retain accuracy is within 0.44 percentage points of Retrain on four benchmarks, with a larger 3.14-point gap on MNIST. Compared with the newly added baselines, GIF improves retain accuracy by 1.38--8.16 points over Freezing and 5.35--13.92 points over Diff-In. GIF achieves lower unlearn accuracy than Diff-In on four datasets, while its unlearn accuracy is 0.20 percentage points higher on PubMed.

The error-bound diagnostics in Appendix~\ref{appendix:error_bound_diagnostics} further show that selectors with positive measured bound gap $G(J)$ retain accuracy better than deliberately poor selectors, and that low target-gradient capture coincides with the largest gap from retraining on MNIST.

\subsection{Selecting Relevant Parameters Preserves Retained Behavior}

\begin{figure*}[htb]
    \centering
    \includegraphics[width=\textwidth]{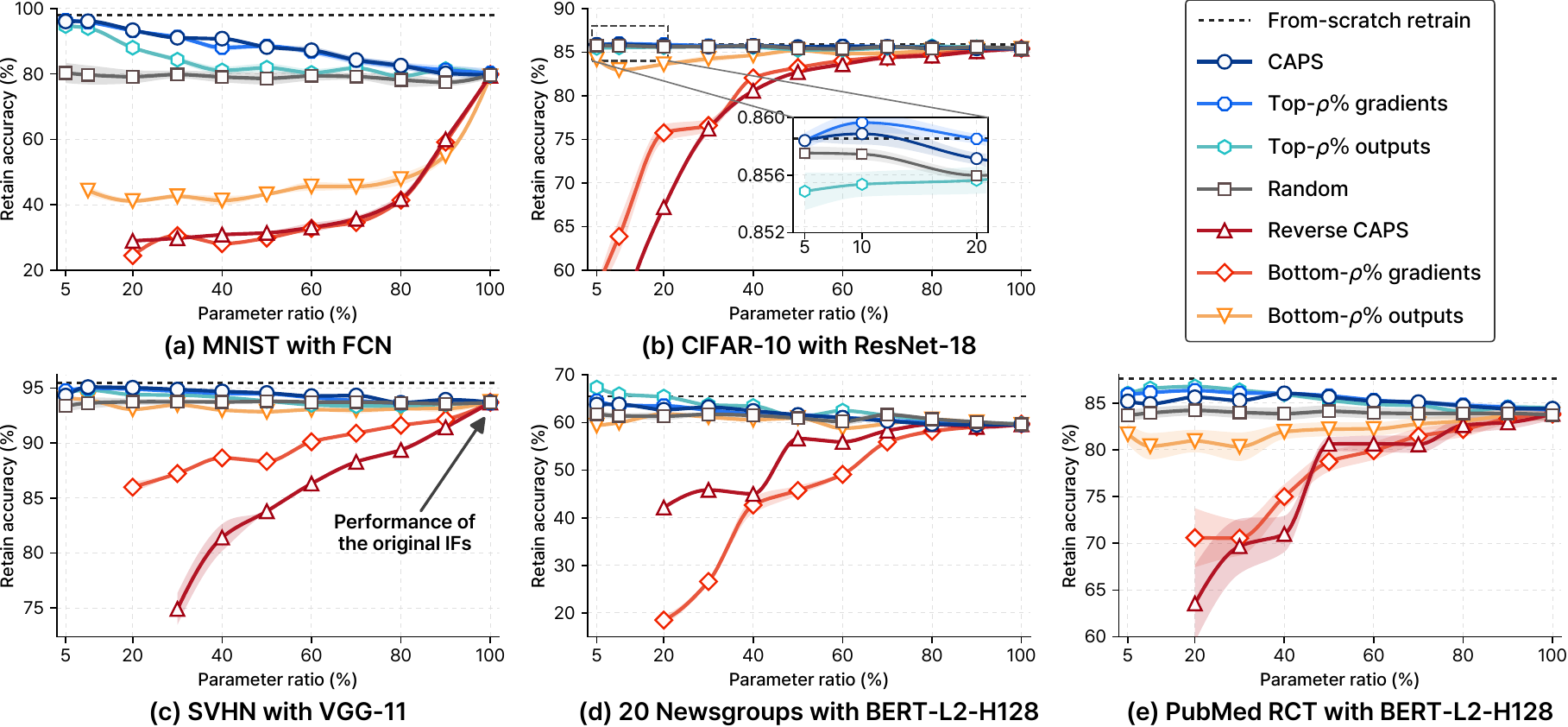}
    \caption{
    Retained accuracy of parameter selection schemes.
    }
    \label{fig:selection_scheme}
\end{figure*}

We next analyze how parameter selection affects model editing.
In this experiment, we keep the influence update rule fixed and vary only the selection method.
We compare the selectors against their counterparts, including Reverse CAPS, Bottom-$\rho\%$ gradients, and Bottom-$\rho\%$ outputs, as well as Random, which selects $\rho\%$ of parameters uniformly at random.
We measure the retained accuracy after the edit on the same target-class removal task.

Figure~\ref{fig:selection_scheme} shows that CAPS and Top-$\rho$ gradients closely track the from-scratch retrained reference at the 5\% and 10\% parameter ratios, achieving nearly identical and stable performance.
Top-$\rho$ outputs performs best on the Newsgroups dataset, but its performance is less stable across architectures and datasets.
In contrast, Reverse CAPS and Bottom-$\rho$ gradients tend to select low-relevance or poorly conditioned directions, and their performance degrades more severely as the parameter budget decreases.
At the 5\% and 10\% ratios, these selectors can even fail to produce class-removal edits.
Interestingly, Bottom-$\rho$ outputs does not show the same severe drop in retained accuracy.
This is because a block with low activation on the target data can still have parameters with sufficient gradient signal for the target loss \cite{huang2025-less, molchanov2017pruning}.

Compared with the full-parameter setting, which corresponds to classical IFs, selecting a relevant parameter subset improves retained accuracy by 0.58 percentage points on CIFAR-10, 16.53 points on MNIST, 7.72 points on Newsgroups, 2.27 points on PubMed 20k RCT, and 1.40 points on SVHN.
These gains show that GIFs can outperform IFs not by increasing the update dimension, but by restricting the edit to parameters that better align with the target behavior while avoiding nuisance changes in retained behavior.
Overall, these results empirically support \textbf{\cref{theorem:Error bound of the GIF,theorem:Self-loss bound of GIFs}} and show that selective model editing is effective when the editable subset is properly chosen.

\subsection{Efficiency and Stability of iHVP Solvers}

\begin{figure*}[htb]
    \centering
    \includegraphics[width=\textwidth]{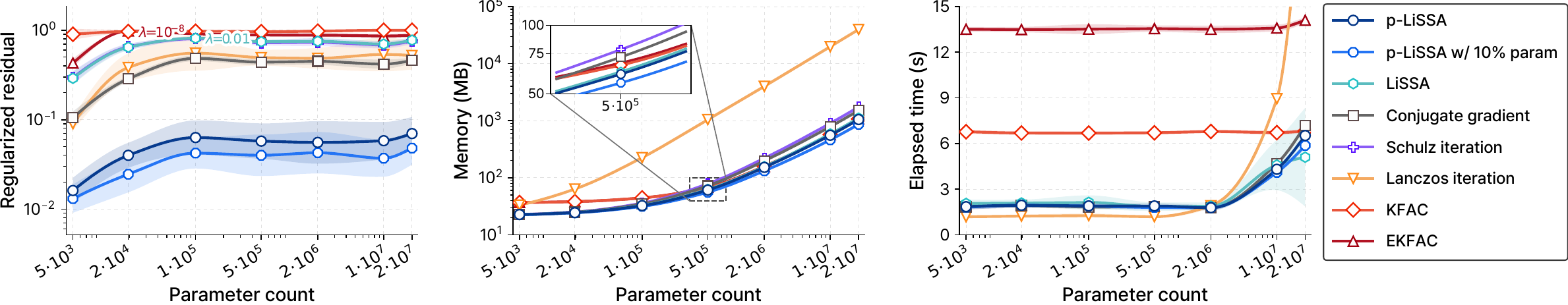}
    \caption{
    Comparison of iHVP solvers across model scales.
    }
    \label{fig:solver_mnist}
\end{figure*}

We compare p-LiSSA with standard iHVP solvers, including LiSSA~\cite{Koh2017_IF}, conjugate gradient, Schulz iteration~\cite{zhou2025-hyperinf}, Lanczos iteration~\cite{schioppa2022-scalingIF}, KFAC, and EKFAC~\cite{grosse2023-ekfac}, across models ranging from 5,000 to 20 million parameters.
For a solver output $\hat{x}$, we measure the relative residual $r(\hat{x})=\|\mathbf{H}\hat{x}-\bm{g}\|_2/\|\bm{g}\|_2$.
Only for p-LiSSA with 10\% parameters, we define the residual as $r(\hat{x})=\|\mathbf{H}_J^\mathsf{T}\mathbf{H}_J\hat{x}-\mathbf{H}_J^\mathsf{T}\bm{g}\|_2/\|\bm{g}\|_2$.
This measures how accurately the computed update satisfies the iHVP linear system, with smaller values indicating a direction closer to the exact inverse-Hessian solution.

Figure~\ref{fig:solver_mnist} shows that p-LiSSA gives the best overall solver trade-off.
At all scales, p-LiSSA obtains an average residual of 0.039.
The best baseline residual is 0.379 from conjugate gradient, so p-LiSSA
reduces the residual by 9.6$\times$.
It also has the lowest average memory use, 227.5 MB compared with 283.0 MB for EKFAC, the most memory-efficient baseline.
The average runtime of p-LiSSA is 2.75 seconds, slightly faster than LiSSA at 2.82 seconds and lower than that of the other solvers.
The same pattern holds at the largest scale.
p-LiSSA with 10\% parameters attains the lowest residual, 0.052, using 864 MB and 5.87 seconds.
LiSSA is slightly faster at 5.10 seconds but has a much larger residual of 0.777, while conjugate gradient reaches only 0.460 residual with 1.51 GB memory.

\section{Discussion}
\label{sec:Limitations, Challenges, and Future Works}

Parameter selection introduces a trade-off between conditioning and target-gradient coverage. Excluding low-curvature directions can improve conditioning, but may also discard components needed to approximate the target response, increasing the projection residual. Effective selection therefore requires considering both quantities together. For large edits, curvature can also change substantially as the model moves away from the reference point, motivating further study of how parameter selection should adapt over the course of an edit.

A remaining computational limitation is that p-LiSSA, although it provides a damping-free restricted solve, still relies on second-order computation.
In particular, computing $H_J^\top H_J I_k$ requires two HVPs, which can become a bottleneck at larger scales.
Developing more efficient restricted curvature approximations and scalable solvers is therefore an important direction for future work.

Another promising direction is to extend GIFs to model editing in generative models. At foundation scale, repeated HVPs and restricted solves can themselves become a bottleneck. Moreover, increasing a training loss does not by itself guarantee semantic forgetting. Evaluation for diffusion models should examine changes in the generated distribution or score field. For LLMs, it should test whether the targeted information remains accessible through paraphrases or alternative reasoning paths. Addressing these computational and evaluation challenges would extend the scope of GIFs and help advance efficient and reliable model editing for generative models.

\ifshowacknowledgments
\section*{Acknowledgments}
This work was supported in part by the Institute of Information \& Communications Technology Planning \& Evaluation (IITP) grant funded by the Korea government (MSIT) (No. RS-2026-25523626, Development of End-to-End Service Performance Guarantee Framework Technology for 6G), in part by the IITP grant funded by the Korea government (MSIT) under the Information Technology Research Center (ITRC) Support Program (No. IITP-2026-2021-0-02048), and in part by the IITP grant funded by the Korea government (MSIT) under the AI Star Fellowship Support Program (No. IITP-2026-RS-2026-25616664).
\fi


\bibliography{references}

@INPROCEEDINGS{Koh2017_IF,
  title     = {Understanding Black-box Predictions via Influence Functions},
  booktitle = {International Conference on Machine Learning},
  author    = {Koh, Pang Wei and Liang, Percy},
  pages     = {1885-1894},
  year      = {2017}
}

@INPROCEEDINGS{Koh2019_GIF,
  title     = {On the Accuracy of Influence Functions for Measuring Group Effects},
  booktitle = {Advances in Neural Information Processing Systems},
  author    = {Koh, Pang Wei and Ang, Kai-Siang and Teo, Hubert H K and Liang, Percy},
  pages     = {5254-5264},
  month     = dec,
  year      = {2019}
}

@INPROCEEDINGS{Basu2021-fragile,
  title     = {Influence Functions in Deep Learning Are Fragile},
  booktitle = {International Conference on Learning Representations},
  author    = {Basu, S and Pope, Phillip E and Feizi, S},
  year      = {2021},
  language  = {en}
}

@article{Agarwal2017-HessianEstimator,
  author  = {Naman Agarwal and Brian Bullins and Elad Hazan},
  title   = {Second-Order Stochastic Optimization for Machine Learning in Linear Time},
  journal = {Journal of Machine Learning Research},
  year    = {2017},
  volume  = {18},
  number  = {116},
  pages   = {1--40},
}

@article{hampel1974-influence,
  title   = {The Influence Curve and Its Role in Robust Estimation},
  author  = {Frank R. Hampel},
  journal = {Journal of the American Statistical Association},
  volume  = {69},
  number  = {346},
  pages   = {383--393},
  year    = {1974}
}

@article{Epifano2023-fragile2,
  title     = {Revisiting the Fragility of Influence Functions},
  journal   = {Neural Networks},
  volume    = {162},
  pages     = {581-588},
  year      = {2023},
  author    = {Jacob R. Epifano and Ravi P. Ramachandran and Aaron J. Masino and Ghulam Rasool}
}

@inproceedings{bae2022-if,
  title     = {If Influence Functions are the Answer, Then What is the Question?},
  author    = {Juhan Bae and Nathan Hoyen Ng and Alston Lo and Marzyeh Ghassemi and Roger Baker Grosse},
  booktitle = {Advances in Neural Information Processing Systems},
  year      = {2022},
}

@inproceedings{devlin2019-bert,
  title     = {BERT: Pre-training of Deep Bidirectional Transformers for Language Understanding},
  author    = {Devlin, Jacob and Chang, Ming-Wei and Lee, Kenton and Toutanova, Kristina},
  booktitle = {Conference of the North American Chapter of the Association for Computational Linguistics},
  month     = jun,
  year      = {2019},
  pages     = {4171--4186}
}

@article{Lanczos1950-eigval_approximation,
  title   = {An Iteration Method for the Solution of the Eigenvalue Problem of Linear Differential and Integral Operators},
  author  = {Cornelius Lanczos},
  journal = {Journal of Research of the National Bureau of Standards},
  year    = {1950},
  volume  = {45},
  pages   = {255-282},
}

@ARTICLE{Lecun1998-MNIST,
  author  = {Lecun, Y. and Bottou, L. and Bengio, Y. and Haffner, P.},
  journal = {Proceedings of the IEEE},
  title   = {Gradient-Based Learning Applied to Document Recognition},
  year    = {1998},
  volume  = {86},
  number  = {11},
  pages   = {2278-2324},
}

@article{Krizhevsky2009-cifar10,
  title   = {Learning Multiple Layers of Features from Tiny Images},
  author  = {Krizhevsky, A. and Hinton, G.},
  journal = {Master's thesis, Department of Computer Science, University of Toronto},
  year    = {2009}
}

@inproceedings{Netzer2011-svhn,
  title     = {Reading Digits in Natural Images with Unsupervised Feature Learning},
  author    = {Yuval Netzer and Tao Wang and Adam Coates and Alessandro Bissacco and Bo Wu and Andrew Y. Ng},
  year      = {2011},
  url       = {http://ufldl.stanford.edu/housenumbers/nips2011_housenumbers.pdf},
  booktitle = {NIPS Workshop on Deep Learning and Unsupervised Feature Learning 2011},
}

@inproceedings{Lang1995-newsweeder,
  year      = {1995},
  author    = {Lang, Ken},
  title     = {NewsWeeder: learning to filter netnews},
  isbn      = {1558603778},
  publisher = {Morgan Kaufmann Publishers Inc.},
  address   = {San Francisco, CA, USA},
  booktitle = {Proceedings of the Twelfth International Conference on International Conference on Machine Learning},
  pages     = {331--339},
  numpages  = {9},
  location  = {Tahoe City, California, USA},
  series    = {ICML'95},
}

@inproceedings{Dernoncourt2017-pubmed-rct,
  title     = {{P}ub{M}ed 200k {RCT}: a Dataset for Sequential Sentence Classification in Medical Abstracts},
  author    = {Dernoncourt, Franck and Lee, Ji Young},
  editor    = {Kondrak, Greg and Watanabe, Taro},
  booktitle = {Proceedings of the Eighth International Joint Conference on Natural Language Processing (Volume 2: Short Papers)},
  month     = nov,
  year      = {2017},
  address   = {Taipei, Taiwan},
  publisher = {Asian Federation of Natural Language Processing},
  url       = {https://aclanthology.org/I17-2052/},
  pages     = {308--313},
}

@article{olah2018-interpretability,
  author  = {Olah, Chris and Satyanarayan, Arvind and Johnson, Ian and Carter, Shan and Schubert, Ludwig and Ye, Katherine and Mordvintsev, Alexander},
  title   = {The Building Blocks of Interpretability},
  journal = {Distill},
  year    = {2018},
  note    = {https://distill.pub/2018/building-blocks},
}

@article{Bau2020-Understanding,
  author  = {David Bau and Jun-Yan Zhu and Hendrik Strobelt and Agata Lapedriza and Bolei Zhou and Antonio Torralba},
  title   = {Understanding the Role of Individual Units in a Deep Neural Network},
  journal = {Proceedings of the National Academy of Sciences of the United States of America},
  volume  = {117},
  number  = {48},
  pages   = {30071-30078},
  year    = {2020},
}

@inproceedings{schioppa2022-scalingIF,
  title     = {Scaling Up Influence Functions},
  author    = {Schioppa, Andrea and Zablotskaia, Polina and Vilar, David and Sokolov, Artem},
  booktitle = {AAAI Conference on Artificial Intelligence},
  volume    = {36},
  year      = {2022},
  pages     = {8179-8186}
}

@inproceedings{Pruthi2020-TracIn,
  author    = {Pruthi, Garima and Liu, Frederick and Kale, Satyen and Sundararajan, Mukund},
  booktitle = {Advances in Neural Information Processing Systems},
  pages     = {19920--19930},
  title     = {Estimating Training Data Influence by Tracing Gradient Descent},
  year      = {2020}
}

@INPROCEEDINGS{Ramprasaath2016-GradCam,
  author    = {Selvaraju, Ramprasaath R. and Cogswell, Michael and Das, Abhishek and Vedantam, Ramakrishna and Parikh, Devi and Batra, Dhruv},
  booktitle = {IEEE/CVF International Conference on Computer Vision},
  title     = {Grad-CAM: Visual Explanations from Deep Networks via Gradient-Based Localization},
  year      = {2017},
  pages     = {618-626},
}

@book{ben2003-MoorePenroseInverse,
  title     = {Generalized Inverses: Theory and Applications},
  author    = {Ben-Israel, Adi and Greville, Thomas N. E.},
  volume    = {15},
  year      = {2003},
  publisher = {Springer}
}

@article{Transtrum11-Levenberg,
  title     = {Geometry of nonlinear least squares with applications to sloppy models and optimization},
  author    = {Transtrum, Mark K. and Machta, Benjamin B. and Sethna, James P.},
  journal   = {Physical Review E},
  volume    = {83},
  issue     = {3},
  pages     = {036701},
  numpages  = {35},
  year      = {2011},
  month     = {Mar},
  publisher = {American Physical Society},
}

@book{Dimitri16-nonlinear_programming,
  title     = {Nonlinear Programming},
  author    = {Bertsekas, Dimitri},
  volume    = {4},
  year      = {2016},
  publisher = {Athena Scientific}
}

@inproceedings{Basu20-SecondIF,
  author    = {Basu, Samyadeep and You, Xuchen and Feizi, Soheil},
  title     = {On second-order group influence functions for black-box predictions},
  year      = {2020},
  publisher = {JMLR.org},
  booktitle = {International Conference on Machine Learning},
  articleno = {67},
  numpages  = {10},
  series    = {ICML'20}
}

@inproceedings{yang2023-dataset,
  title     = {Dataset Pruning: Reducing Training Data by Examining Generalization Influence},
  author    = {Shuo Yang and Zeke Xie and Hanyu Peng and Min Xu and Mingming Sun and Ping Li},
  booktitle = {International Conference on Learning Representations},
  year      = {2023},
}

@article{Wu22-PUMA,
  title    = {{PUMA}: Performance Unchanged Model Augmentation for Training Data Removal},
  volume   = {36},
  number   = {8},
  journal  = {Proceedings of the AAAI Conference on Artificial Intelligence},
  author   = {Wu, Ga and Hashemi, Masoud and Srinivasa, Christopher},
  year     = {2022},
  month    = {Jun.},
  pages    = {8675-8682}
}

@inproceedings{Wu23-GraphUnlearning,
  author    = {Wu, Jiancan and Yang, Yi and Qian, Yuchun and Sui, Yongduo and Wang, Xiang and He, Xiangnan},
  title     = {{GIF}: A General Graph Unlearning Strategy via Influence Function},
  year      = {2023},
  isbn      = {9781450394161},
  publisher = {Association for Computing Machinery},
  address   = {New York, NY, USA},
  booktitle = {Proceedings of the ACM Web Conference},
  pages     = {651–661},
  numpages  = {11}
}

@inproceedings{kong2022-Relabeling,
  title     = {Resolving Training Biases via Influence-based Data Relabeling},
  author    = {Shuming Kong and Yanyan Shen and Linpeng Huang},
  booktitle = {International Conference on Learning Representations},
  year      = {2022},
}

@InProceedings{Gong2022-LabelInfluence,
  title     = {Partial Label Learning via Label Influence Function},
  author    = {Gong, Xiuwen and Yuan, Dong and Bao, Wei},
  booktitle = {International Conference on Machine Learning},
  pages     = {7665--7678},
  year      = {2022},
  editor    = {Chaudhuri, Kamalika and Jegelka, Stefanie and Song, Le and Szepesvari, Csaba and Niu, Gang and Sabato, Sivan},
  volume    = {162},
  series    = {Proceedings of Machine Learning Research},
  month     = {17--23 Jul},
  publisher = {PMLR},
}

@inproceedings{kwon2024datainf,
  title     = {DataInf: Efficiently Estimating Data Influence in LoRA-tuned LLMs and Diffusion Models},
  author    = {Yongchan Kwon and Eric Wu and Kevin Wu and James Zou},
  booktitle = {International Conference on Learning Representations},
  year      = {2024},
}

@INPROCEEDINGS{yoon2023fewshot,
  author    = {Yoon, Youngsik and Nam, Jinhwan and Yun, Hyojeong and Lee, Jaeho and Kim, Dongwoo and Ok, Jungseul},
  booktitle = {IEEE Symposium on Security and Privacy},
  title     = {Few-shot Unlearning},
  year      = {2024},
  pages     = {3276-3292},
  publisher = {IEEE Computer Society},
  address   = {Los Alamitos, CA, USA},
  month     = {May}
}

@book{Stewart98-Neumann,
  author    = {Stewart, G. W.},
  title     = {Matrix Algorithms},
  publisher = {Society for Industrial and Applied Mathematics},
  year      = {1998},
}

@misc{das2024-pruning,
  title         = {Beyond Size: How Gradients Shape Pruning Decisions in Large Language Models},
  author        = {Rocktim Jyoti Das and Mingjie Sun and Liqun Ma and Zhiqiang Shen},
  year          = {2024},
  eprint        = {2311.04902},
  archivePrefix = {arXiv},
  primaryClass  = {cs.CL},
}

@ARTICLE{Shi19-pruning,
  author   = {Chen, Shi and Zhao, Qi},
  journal  = {IEEE Transactions on Pattern Analysis and Machine Intelligence},
  title    = {Shallowing Deep Networks: Layer-Wise Pruning Based on Feature Representations},
  year     = {2019},
  volume   = {41},
  number   = {12},
  pages    = {3048-3056},
  doi      = {10.1109/TPAMI.2018.2874634}
}

@inproceedings{sung2024-pruning,
  title     = {ECoFLaP: Efficient Coarse-to-Fine Layer-Wise Pruning for Vision-Language Models},
  author    = {Yi-Lin Sung and Jaehong Yoon and Mohit Bansal},
  booktitle = {International Conference on Learning Representations},
  year      = {2024},
}

@inproceedings{de2021-editing,
    title = "Editing Factual Knowledge in Language Models",
    author = "De Cao, Nicola  and
      Aziz, Wilker  and
      Titov, Ivan",
    editor = "Moens, Marie-Francine  and
      Huang, Xuanjing  and
      Specia, Lucia  and
      Yih, Scott Wen-tau",
    booktitle = "Proceedings of the 2021 Conference on Empirical Methods in Natural Language Processing",
    month = nov,
    year = "2021",
    address = "Online and Punta Cana, Dominican Republic",
    publisher = "Association for Computational Linguistics",
    url = "https://aclanthology.org/2021.emnlp-main.522/",
    doi = "10.18653/v1/2021.emnlp-main.522",
    pages = "6491--6506",
}

@inproceedings{
Sinitsin2020-Editable,
title={Editable Neural Networks},
author={Anton Sinitsin and Vsevolod Plokhotnyuk and Dmitry Pyrkin and Sergei Popov and Artem Babenko},
booktitle={International Conference on Learning Representations},
year={2020},
url={https://openreview.net/forum?id=HJedXaEtvS}
}

@inproceedings{
rubinstein2025-rescaled,
title={Rescaled Influence Functions: Accurate Data Attribution in High Dimension},
author={Ittai Rubinstein and Samuel B. Hopkins},
booktitle={The Thirty-ninth Annual Conference on Neural Information Processing Systems},
year={2025},
url={https://openreview.net/forum?id=0T9lN3m3pc}
}

@misc{Yegor2024-iHVP,
      title={Revisiting inverse Hessian vector products for calculating influence functions}, 
      author={Yegor Klochkov and Yang Liu},
      year={2024},
      eprint={2409.17357},
      archivePrefix={arXiv},
      primaryClass={cs.LG},
      url={https://arxiv.org/abs/2409.17357}, 
}

@inproceedings{
hong2025-better,
title={Better Hessians Matter: Studying the Impact of Curvature Approximations in Influence Functions},
author={Dat Minh Hong and Bruno Kacper Mlodozeniec and Runa Eschenhagen and Richard E. Turner},
booktitle={Mechanistic Interpretability Workshop at NeurIPS 2025},
year={2025},
url={https://openreview.net/forum?id=383i8264it}
}

@InProceedings{Alaa20-Higher-OrderInfluenceFunctions,
  title = 	 {Discriminative Jackknife: Quantifying Uncertainty in Deep Learning via Higher-Order Influence Functions},
  author =       {Alaa, Ahmed and Van Der Schaar, Mihaela},
  booktitle = 	 {Proceedings of the 37th International Conference on Machine Learning},
  pages = 	 {165--174},
  year = 	 {2020},
  editor = 	 {III, Hal Daumé and Singh, Aarti},
  volume = 	 {119},
  series = 	 {Proceedings of Machine Learning Research},
  month = 	 {13--18 Jul},
  publisher =    {PMLR},
  url = 	 {https://proceedings.mlr.press/v119/alaa20a.html},
}

@misc{
zhou2025-hyperinf,
title={Hyper{INF}: Unleashing the HyperPower of the Schulz's Method for Data Influence Estimation},
author={Xinyu Zhou and Simin Fan and Martin Jaggi},
year={2025},
url={https://openreview.net/forum?id=OLtD2vDF5X}
}

@inproceedings{
ye2025-flatvm,
title={Towards Robust Influence Functions with Flat Validation Minima},
author={Xichen Ye and Yifan Wu and Weizhong Zhang and Cheng Jin and Yifan Chen},
booktitle={Forty-second International Conference on Machine Learning},
year={2025},
url={https://openreview.net/forum?id=Um7XmQEWu5}
}

@book{Horn2012-matrixanalysis,
author = {Horn, Roger A. and Johnson, Charles R.},
title = {Matrix Analysis},
year = {2012},
isbn = {0521548233},
publisher = {Cambridge University Press},
address = {USA},
edition = {2nd},
}

@InProceedings{martens15-kfac,
  title = 	 {Optimizing Neural Networks with {K}ronecker-factored Approximate Curvature},
  author = 	 {Martens, James and Grosse, Roger},
  booktitle = 	 {Proceedings of the 32nd International Conference on Machine Learning},
  pages = 	 {2408--2417},
  year = 	 {2015},
  editor = 	 {Bach, Francis and Blei, David},
  volume = 	 {37},
  series = 	 {Proceedings of Machine Learning Research},
  address = 	 {Lille, France},
  month = 	 {07--09 Jul},
  publisher =    {PMLR},
  url = 	 {https://proceedings.mlr.press/v37/martens15.html},
}

@misc{grosse2023-ekfac,
      title={Studying Large Language Model Generalization with Influence Functions}, 
      author={Roger Grosse and Juhan Bae and Cem Anil and Nelson Elhage and Alex Tamkin and Amirhossein Tajdini and Benoit Steiner and Dustin Li and Esin Durmus and Ethan Perez and Evan Hubinger and Kamilė Lukošiūtė and Karina Nguyen and Nicholas Joseph and Sam McCandlish and Jared Kaplan and Samuel R. Bowman},
      year={2023},
      eprint={2308.03296},
      archivePrefix={arXiv},
      primaryClass={cs.LG},
      url={https://arxiv.org/abs/2308.03296}, 
}

@inproceedings{
lin2018deep,
title={Deep Gradient Compression: Reducing the Communication Bandwidth for Distributed Training},
author={Yujun Lin and Song Han and Huizi Mao and Yu Wang and Bill Dally},
booktitle={International Conference on Learning Representations},
year={2018},
url={https://openreview.net/forum?id=SkhQHMW0W},
}

@InProceedings{sun17c-meProp,
  title = 	 {me{P}rop: Sparsified Back Propagation for Accelerated Deep Learning with Reduced Overfitting},
  author =       {Xu Sun and Xuancheng Ren and Shuming Ma and Houfeng Wang},
  booktitle = 	 {Proceedings of the 34th International Conference on Machine Learning},
  pages = 	 {3299--3308},
  year = 	 {2017},
  editor = 	 {Precup, Doina and Teh, Yee Whye},
  volume = 	 {70},
  series = 	 {Proceedings of Machine Learning Research},
  month = 	 {06--11 Aug},
  publisher =    {PMLR},
  url = 	 {https://proceedings.mlr.press/v70/sun17c.html},
}

@inproceedings{huang2025-less,
    title = "Less Is More? Examining Fairness in Pruned Large Language Models for Summarising Opinions",
    author = "Huang, Nannan  and
      Fayek, Haytham M.  and
      Zhang, Xiuzhen",
    editor = "Christodoulopoulos, Christos  and
      Chakraborty, Tanmoy  and
      Rose, Carolyn  and
      Peng, Violet",
    booktitle = "Proceedings of the 2025 Conference on Empirical Methods in Natural Language Processing",
    month = nov,
    year = "2025",
    address = "Suzhou, China",
    publisher = "Association for Computational Linguistics",
    url = "https://aclanthology.org/2025.emnlp-main.909/",
    doi = "10.18653/v1/2025.emnlp-main.909",
    pages = "17994--18018",
    ISBN = "979-8-89176-332-6",
}

@inproceedings{
molchanov2017pruning,
title={Pruning Convolutional Neural Networks for Resource Efficient Inference},
author={Pavlo Molchanov and Stephen Tyree and Tero Karras and Timo Aila and Jan Kautz},
booktitle={International Conference on Learning Representations},
year={2017},
url={https://openreview.net/forum?id=SJGCiw5gl}
}

@InProceedings{giordano2019swiss,
  title = 	 {A Swiss Army Infinitesimal Jackknife},
  author =       {Giordano, Ryan and Stephenson, William and Liu, Runjing and Jordan, Michael and Broderick, Tamara},
  booktitle = 	 {Proceedings of the Twenty-Second International Conference on Artificial Intelligence and Statistics},
  pages = 	 {1139--1147},
  year = 	 {2019},
  editor = 	 {Chaudhuri, Kamalika and Sugiyama, Masashi},
  volume = 	 {89},
  series = 	 {Proceedings of Machine Learning Research},
  month = 	 {16--18 Apr},
  publisher =    {PMLR},
  url = 	 {https://proceedings.mlr.press/v89/giordano19a.html},
}

@inproceedings{li2025do,
    title = "Do Influence Functions Work on Large Language Models?",
    author = "Li, Zhe  and
      Zhao, Wei  and
      Li, Yige  and
      Sun, Jun",
    editor = "Christodoulopoulos, Christos  and
      Chakraborty, Tanmoy  and
      Rose, Carolyn  and
      Peng, Violet",
    booktitle = "Findings of the Association for Computational Linguistics: EMNLP 2025",
    month = nov,
    year = "2025",
    address = "Suzhou, China",
    publisher = "Association for Computational Linguistics",
    url = "https://aclanthology.org/2025.findings-emnlp.775/",
    doi = "10.18653/v1/2025.findings-emnlp.775",
    pages = "14367--14382",
    ISBN = "979-8-89176-335-7",
}

@inproceedings{hara2019cleansing,
  title={{Data Cleansing for Models Trained with SGD}},
  author={Hara, Satoshi and Nitanda, Atsushi and Maehara, Takanori},
  booktitle={Advances in Neural Information Processing Systems},
  year={2019}
}

@inproceedings{tan2023moso,
  title={{Data Pruning via Moving-one-Sample-out}},
  author={Tan, Haoru and Wu, Sitong and Du, Fei and Chen, Yukang and Wang, Zhibin and Wang, Fan and Qi, Xiaojuan},
  booktitle={Advances in Neural Information Processing Systems},
  year={2023}
}

@inproceedings{wang2025temporal,
  title={{Capturing the Temporal Dependence of Training Data Influence}},
  author={Wang, Jiachen T. and Song, Dawn and Zou, James and Mittal, Prateek and Jia, Ruoxi},
  booktitle={International Conference on Learning Representations},
  year={2025}
}

@article{tan2025diffin,
  title={{Understanding Data Influence with Differential Approximation}},
  author={Tan, Haoru and Wu, Sitong and Wu, Xiuzhe and Wang, Wang and Zhao, Bo and Xie, Zeke and Xia, Gui-Song and Qi, Xiaojuan},
  journal={arXiv preprint arXiv:2508.14648},
  year={2025}
}

@inproceedings{tan2025rftinf,
  title={{Understanding Data Influence in Reinforcement Finetuning}},
  author={Tan, Haoru and Wu, Xiuzhe and Wu, Sitong and Zhang, Shaofeng and Chen, Yanfeng and Sun, Xingwu and Shen, Jeanne and Qi, Xiaojuan},
  booktitle={Advances in Neural Information Processing Systems},
  year={2025}
}
\bibliographystyle{plain}

\newpage
\appendix
\onecolumn

\section{Related Works}
\label{sec:Related Works}

\begin{table}[h]
\centering
\caption{Comparison of representative influence-function methods.
The methods are organized by problem setting, IF variant, assumptions, estimator, computational profile, empirical behavior, and limitations.}
\label{tab:if_related_work_comparison}
\resizebox{\textwidth}{!}{
\begin{tabular}{p{2.6cm} p{2.8cm} p{2.6cm} p{3.2cm} p{3.0cm} p{2.9cm} p{3.1cm} p{3.2cm}}
\toprule
\textbf{Method} &
\textbf{Problem setting} &
\textbf{IF variant} &
\textbf{Main assumptions} &
\textbf{Estimator / computation} &
\textbf{Cost profile} &
\textbf{Empirical behavior} &
\textbf{Limitations} \\
\midrule

\multicolumn{8}{l}{\textbf{Classical and higher-order influence-function variants}} \\
\midrule

Classical IF~\cite{Koh2017_IF} &
Single-example attribution for test loss or prediction &
First-order IF &
Local twice differentiability; invertible or damped Hessian; infinitesimal reweighting &
IHVP with HVP-based iterative solvers &
Requires one or more IHVP solves per query &
Effective for convex models and some small neural networks &
Unstable in deep non-convex models; updates the full parameter space \\

Group IF~\cite{Koh2019_GIF} &
Group removal and leave-$k$-out approximation &
First-order group IF &
Small-perturbation approximation; group structure preserves useful ranking &
Sum of pointwise IFs or group-specific influence computation &
Efficient when additive; self-loss requires additional computation &
Can preserve ranking even when absolute error is large &
Large coherent groups can induce substantial approximation error \\

Second-order group IF~\cite{Basu20-SecondIF} &
Selecting influential groups for black-box predictions &
Second-order IF &
Second-order Taylor approximation is informative; higher-order derivatives or approximations are tractable &
First-order term plus second-order correction &
More expensive than first-order IF &
Improves correlation in linear models and helps in some neural-network settings &
Still sensitive in deep non-convex regimes; higher computation cost \\

Higher-order IF / DJ~\cite{Alaa20-Higher-OrderInfluenceFunctions} &
Uncertainty quantification in deep regression &
Higher-order IF &
Von Mises expansion; repeated higher-order gradient computation is feasible &
Higher-order IF approximation to leave-one-out behavior &
High computational cost &
Competitive discriminative intervals compared with Bayesian alternatives &
Focuses on uncertainty quantification rather than direct data attribution or editing \\

\midrule

\multicolumn{8}{l}{\textbf{Scalable influence estimators for modern models}} \\
\midrule

TracIn~\cite{Pruthi2020-TracIn} &
Large-scale data attribution &
Checkpoint-based first-order surrogate &
Training checkpoints provide useful trajectory information &
Cumulative gradient inner products over checkpoints &
Linear in the number of checkpoints; no Hessian inverse &
Highly practical and scalable &
Weak direct connection to retraining or curvature-corrected Newton dynamics \\

Scaling up IF~\cite{schioppa2022-scalingIF} &
Influence estimation for large transformer models &
Arnoldi-based low-rank IHVP &
Low-dimensional Krylov approximation is sufficient &
Arnoldi basis with low-rank inverse approximation &
High upfront basis construction cost; efficient reuse afterward &
Scales to large language and vision transformers &
Approximation quality depends on rank and Hessian spectrum \\

DataInf~\cite{kwon2024datainf} &
Data influence in LoRA-tuned LLMs and diffusion models &
Closed-form approximate IF &
Low effective dimensionality induced by PEFT / LoRA &
Closed-form influence estimator &
Low memory and computation &
Fast and accurate in PEFT settings &
Less directly applicable to full fine-tuning \\

\midrule

\multicolumn{8}{l}{\textbf{Robust and structured influence estimators}} \\
\midrule

HyperINF~\cite{zhou2025-hyperinf} &
Influence estimation for LoRA-tuned LLMs and VLMs &
Schulz-type hyperpower IF &
Low-rank GFIM or Hessian surrogate is accurate &
Schulz iteration with structured curvature approximation &
Nearly independent of LoRA rank in favorable settings &
Improves accuracy, stability, and memory efficiency &
Depends on surrogate quality and PEFT structure \\

\bottomrule
\end{tabular}
}
\end{table}

\begin{table}[h]
\centering
\caption{Chronological development of influence-function research.}
\label{tab:if_timeline}
\resizebox{\textwidth}{!}{
\begin{tabular}{p{0.7cm} p{4.4cm} p{9.8cm}}
\toprule
\textbf{Year} & \textbf{Representative work} & \textbf{Main contribution} \\
\midrule
1974 & Hampel~\cite{hampel1974-influence} &
Formalized the influence curve as a sensitivity measure under infinitesimal contamination. \\
\arrayrulecolor{lightgray}
\midrule

2017 & Agarwal et al.~\cite{Agarwal2017-HessianEstimator}; Koh and Liang~\cite{Koh2017_IF} &
Introduced LiSSA for stochastic second-order optimization and reintroduced IFs to modern machine learning through IHVP-based data attribution. \\
\midrule

2019 & Giordano et al.~\cite{giordano2019swiss}; Koh et al.~\cite{Koh2019_GIF} &
Developed finite-sample infinitesimal-jackknife guarantees and analyzed group effects for data removal. \\
\midrule

2020 & Basu et al.~\cite{Basu20-SecondIF}; Alaa and van der Schaar~\cite{Alaa20-Higher-OrderInfluenceFunctions}; Pruthi et al.~\cite{Pruthi2020-TracIn}; Sinitsin et al.~\cite{Sinitsin2020-Editable} &
Extended IFs to second-order group effects and higher-order uncertainty estimation; proposed scalable gradient-tracing attribution and editable neural networks. \\
\midrule

2021 & Basu et al.~\cite{Basu2021-fragile}; De Cao et al.~\cite{de2021-editing} &
Empirically demonstrated IF fragility in deep learning and studied factual knowledge editing in language models. \\
\midrule

2022 & Bae et al.~\cite{bae2022-if}; Schioppa et al.~\cite{schioppa2022-scalingIF}; Wu et al.~\cite{Wu22-PUMA}; Kong et al.~\cite{kong2022-Relabeling} &
Analyzed the mismatch between IFs and retraining, scaled IFs with Arnoldi iteration, and applied influence-based ideas to data removal and relabeling. \\
\midrule

2023 & Epifano et al.~\cite{Epifano2023-fragile2}; Wu et al.~\cite{Wu23-GraphUnlearning}; Grosse et al.~\cite{grosse2023-ekfac} &
Revisited IF fragility, proposed influence-function-based graph unlearning, and used IFs to study large language model generalization. \\
\midrule

2024 & Kwon et al.~\cite{kwon2024datainf}; Klochkov and Liu~\cite{Yegor2024-iHVP} &
Proposed efficient influence estimation for LoRA-tuned models and revisited practical IHVP computation. \\
\midrule

2025 & Ye et al.~\cite{ye2025-flatvm}; Rubinstein and Hopkins~\cite{rubinstein2025-rescaled}; Zhou et al.~\cite{zhou2025-hyperinf}; Li et al.~\cite{li2025do} &
Developed robust and rescaled IF variants, Schulz-type large-scale influence estimators, and analyzed the limitations of IFs in LLMs. \\

\arrayrulecolor{black}
\bottomrule
\end{tabular}
}
\end{table}

This paper focuses on improving inaccurate IF estimates for model editing.
We organize the review into five areas covering the statistical foundations and higher-order extensions of IFs, their fragility in deep non-convex models, scalable IHVP and curvature approximations, applications to editing, unlearning, and relabeling, and trajectory-aware data influence.

\paragraph{Statistical foundations and higher-order influence estimates.}
IFs originated in robust statistics as a local sensitivity measure of a statistical functional under infinitesimal contamination~\cite{hampel1974-influence}.
In modern machine learning, IFs were reintroduced for empirical risk minimization, where the effect of upweighting or removing a training example on a test loss can be approximated through an inverse-Hessian-vector product~\cite{Koh2017_IF}.
This formulation typically relies on local twice differentiability, invertibility or appropriate damping of the Hessian, and a small-perturbation approximation.

Subsequent work has extended the classical first-order formulation to larger perturbations and more structured data edits.
Group effects were analyzed to show that, although absolute errors may grow for larger groups, influence estimates can still preserve useful ranking information~\cite{Koh2019_GIF}.
Finite-sample guarantees were also developed for infinitesimal jackknife approximations to leave-$k$-out retraining~\cite{giordano2019swiss}.
Higher-order variants further attempt to reduce linearization error.
Second-order group influence functions were proposed for black-box prediction, and higher-order influence-function estimators were developed for uncertainty quantification in deep learning~\cite{Basu20-SecondIF,Alaa20-Higher-OrderInfluenceFunctions}.

\paragraph{Fragility of influence functions in deep models.}
Despite their theoretical appeal, classical IFs often become unreliable in deep non-convex settings.
Prior work empirically showed that IF accuracy is highly sensitive to model depth, width, regularization, test-point choice, and stochastic IHVP estimation~\cite{Basu2021-fragile}.
Other work decomposed the mismatch between IFs and retraining into multiple error sources and argued that practical IF implementations often approximate a proximal Bregman response function rather than exact leave-one-out retraining~\cite{bae2022-if}.
Further studies revisited the fragility of IFs and showed that some reported failures can depend on the evaluation protocol~\cite{Epifano2023-fragile2}.

These studies suggest that the difficulty is not only statistical but also geometric and numerical.
Classical IFs compute an update in the full parameter space, where low-curvature or poorly conditioned directions can amplify small gradient components and induce parameter changes that are weakly related to the edited data.
This observation motivates our restricted formulation, which explicitly selects an editable parameter subspace based on data relevance and curvature stability.

\paragraph{Scalable IHVP and curvature approximations.}
A major computational bottleneck in IF estimation is the IHVP.
LiSSA is a widely used iterative estimator based on a Neumann-series approximation, but it typically requires scaling and damping for numerical stability in modern neural networks~\cite{Agarwal2017-HessianEstimator,Koh2017_IF}.
Several alternatives have been proposed to improve scalability or stability.
Arnoldi iteration has been used to scale IFs to large transformer models~\cite{schioppa2022-scalingIF}.
EK-FAC-style curvature approximations have been employed for studying generalization in large language models~\cite{grosse2023-ekfac}.
Closed-form influence estimators have been proposed for LoRA-tuned large language models and diffusion models, and Schulz-type hyperpower iterations have been developed for efficient influence estimation in parameter-efficient fine-tuning settings~\cite{kwon2024datainf,zhou2025-hyperinf}.
Recent work has also emphasized that the quality of IHVP and curvature approximations can dominate the final accuracy of influence estimates~\cite{Yegor2024-iHVP,hong2025-better}.

Most of these methods aim to approximate the full-parameter IHVP more efficiently or more accurately. In particular, Schioppa et al.~\cite{schioppa2022-scalingIF} approximate the conventional full-parameter IF in a dominant spectral/Krylov subspace, which can mitigate instability associated with poorly represented spectral directions. GIF similarly benefits from avoiding low-curvature directions, but defines a finite model-edit update on a target-dependent coordinate support. Its selection condition also requires target-gradient capture, quantified by the projection residual, rather than curvature alone.
In contrast, our approach changes the object being solved.
GIFs define a restricted inverse-sub-Hessian problem on an editable parameter subspace, and p-LiSSA solves this restricted problem without damping-induced bias.
This distinction is important because the restricted problem can avoid irrelevant or ill-conditioned directions before the numerical IHVP computation is performed.

\paragraph{Influence functions for editing, unlearning, and relabeling.}
Beyond data attribution and debugging, IFs have been used in model editing and unlearning.
TracIn traces gradient inner products along training checkpoints and provides a scalable first-order surrogate for data attribution, although it does not directly approximate Newton-style retraining dynamics~\cite{Pruthi2020-TracIn}.
Editable neural networks and KnowledgeEditor study local edits to model behavior~\cite{Sinitsin2020-Editable,de2021-editing}.
For data removal, influence-based updates have been proposed for performance-preserving training-data removal and graph unlearning~\cite{Wu22-PUMA,Wu23-GraphUnlearning}.
Influence estimates have also been used for relabeling and label correction~\cite{kong2022-Relabeling,Gong2022-LabelInfluence}.
In large-scale language-model settings, IFs have been used to analyze generalization and data influence~\cite{grosse2023-ekfac,kwon2024datainf}, but recent studies suggest that parameter-space changes, model-behavior changes, IHVP approximation errors, and fine-tuning convergence issues can cause standard IFs to fail in LLMs~\cite{li2025do}.

\paragraph{Trajectory-aware data influence.}
Several methods estimate sample attribution or utility by using intermediate checkpoints and the optimizer trajectory, including SGD retracing for data cleansing~\cite{hara2019cleansing}, moving-one-sample-out pruning~\cite{tan2023moso}, trajectory-specific influence embeddings~\cite{wang2025temporal}, differential influence accumulation~\cite{tan2025diffin}, and influence estimation for reinforcement fine-tuning~\cite{tan2025rftinf}. These methods retain information about how individual examples interact with the evolving training state. GIF constructs one joint, class-level finite edit from the final checkpoint. Consequently, summing sample-wise trajectory influences is not generally equivalent to the class-level edit considered here; neither formulation removes the need for the other.

Overall, prior work has made substantial progress in theoretical influence estimation, scalable curvature approximation, and downstream applications.
However, relatively less attention has been paid to finite model edits performed in a restricted, data-relevant, and numerically stable parameter subspace.
GIFs and p-LiSSA address this gap by jointly formulating a restricted edit direction and a damping-free solver for the corresponding inverse-sub-Hessian problem.

\section{The Existence of Negative Eigenvalues}
\label{Appendix:The Existence of Negative Eigenvalues}

\begin{figure}[ht]
    \centering
    \includegraphics[width=0.65\linewidth]{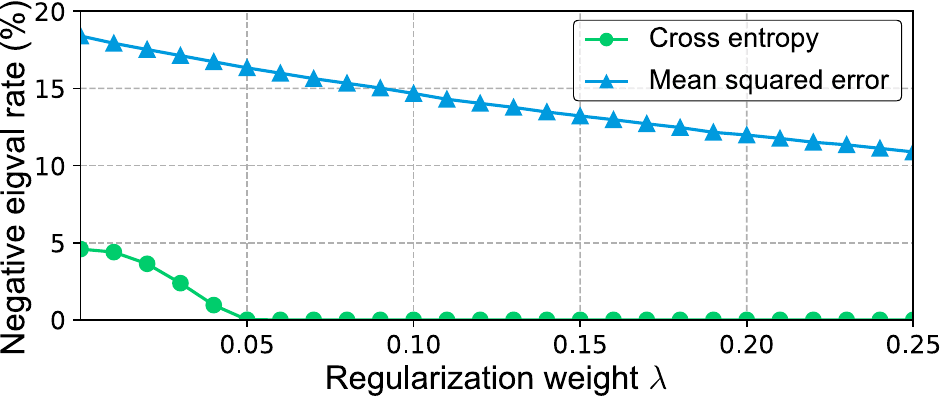}
    \caption{Negative eigenvalue ratio under various regularization weights $\lambda$ for the MNIST classification on the fully-connected network.
    }
    \label{fig:neg_eigval_rate}
\end{figure}

Fig.\ref{fig:neg_eigval_rate} shows the percentage of negative eigenvalues among the top 500 eigenvalues with the largest absolute values in the Hessian matrix, obtained using the Lanczos algorithm \cite{Lanczos1950-eigval_approximation}.
The negative eigenvalue rate for the cross-entropy loss implies that the Hessian matrix is positive semi-definite for large enough $\lambda$.
However, the network might not be trained with a sufficiently large $\lambda$ to ensure all eigenvalues are positive.
Moreover, such $\lambda$ may exist outside of the feasible range as in the case of the mean squared error.

\section{GIF and Freezing on Selected Parameters}
\label{appendix:gif_freezing}

GIF and parameter freezing answer different linearized problems when only the coordinates in $J$ are editable.
With fixed coordinates held constant, the linearized first-order condition of the restricted empirical-risk problem is
\begin{align}
    \bm 0
    =
    \epsilon\bm g_{\bm w,J}
    +
    \mathbf H_{JJ}\Delta_J^{\mathsf{Freeze}},
\end{align}
and hence
\begin{align}
    \Delta_J^{\mathsf{Freeze}}
    =
    -\epsilon\mathbf H_{JJ}^{-1}\bm g_{\bm w,J}.
    \label{eq:freezing_update}
\end{align}
Here, $\bm g_{\bm w,J}$ and $\mathbf H_{JJ}$ are the target gradient and Hessian restricted to the editable coordinates.
This is the usual freezing update, provided $\mathbf H_{JJ}$ is invertible. It approximates the solution of the restricted ERM problem to first order.
An exact minimizer of the restricted ERM objective attains the smallest loss among all updates supported on the same $J$, including GIF.

GIF instead minimizes the residual of the full-model linearized condition after constraining the update to $J$. When $\mathbf H_J$ has full column rank, its update is
\begin{align}
    \Delta_J^{\mathsf{GIF}}
    &=
    \operatorname*{argmin}_{\Delta_J}
    \left\|
        \epsilon\bm g_{\bm w}+\mathbf H_J\Delta_J
    \right\|_2^2
    =
    -\epsilon\mathbf H_J^+\bm g_{\bm w}.
\end{align}
Writing $\mathbf H_J=[\mathbf H_{JJ}^{\mathsf T}\mid
\mathbf H_{-J,J}^{\mathsf T}]^{\mathsf T}$ makes the distinction explicit.
GIF uses the full target gradient and includes the change induced in the fixed-coordinate gradients through $\mathbf H_{-J,J}$, whereas freezing uses only the editable-coordinate optimality condition.
If $\mathbf H_{-J,J}=\mathbf 0$, the two updates agree.
These formulations target different criteria, and the experiments compare their resulting finite model edits.

\section{Proof of Theorem~\ref{theorem:Error bound of the GIF}}
\label{appendix:proof_error_bound}

\begin{tcolorbox}[
width=1.00\linewidth, colframe=box_margin, colback=box_background, boxrule=0.3mm, boxsep=-3pt, left=8pt, right=8pt]
\textbf{Theorem~\ref{theorem:Error bound of the GIF}} (Error bound of GIFs)\textbf{.}
    Let the functional $f(\cdot)$ and the Hessian map $\nabla^2\mathcal{L}_{\bm{1}-\bm{w}}(\cdot)$ be $C_f$- and $C_\mathbf{H}$-Lipschitz continuous on a convex neighborhood $\mathcal U$ containing $\boldthetahat$, $\boldthetahat^{\mathsf{GIF}}(\bm{1}-\bm{w})$, and the retrained optimum $\boldthetahat(\bm{1}-\bm{w})$.
    Assume $\nabla\mathcal L_{\bm 1}(\boldthetahat)=\bm 0$.
    If $\mathcal{L}_{\bm{1}-\bm{w}}$ is $\sigma_{\min}(\mathbf H_{\bm 1-\bm w})$-strongly convex on $\mathcal U$, with $\sigma_{\min}(\mathbf H_{\bm 1-\bm w})>0$,
    then, for $\bm\Pi_J=\mathbf H_J\mathbf H_J^+$ and $\rho(J)=\|(\mathbf I-\bm\Pi_J)\bm g_{\bm w}\|_2$, the evaluation error is upper-bounded as
    \begin{equation}
        \resizebox{\dimexpr\linewidth-4.5em\relax}{!}{$\displaystyle
        \big|\mathrm{Err}_{J}(\bm{w}\mid\boldthetahat)\big|
        \le
        \frac{C_f}{\sigma_{\min}(\mathbf H_{\bm 1-\bm w})}
        \left[
            \frac{C_\mathbf H}{2}\|\mathbf H_J^+\bm g_{\bm w}\|_2^2
            +\rho(J)
        \right]
        \le
        \frac{C_f}{\sigma_{\min}(\mathbf H_{\bm 1-\bm w})}
        \left[
            \frac{C_\mathbf H}{2}
            \frac{\|\bm g_{\bm w}\|_2^2}{\sigma_{\min}(\mathbf H_J)^2}
            +\rho(J)
        \right].$}
    \end{equation}
\end{tcolorbox}

This inequality for IFs has been earlier studied by \cite{Koh2019_GIF}, and we extend the error bound to GIFs.
Define $\mathcal{L}_{\bm{1}-\bm{w}}(\boldtheta) =\frac{1}{m} \sum_{i=1}^m (1-w_i) \ell(z_i,\boldtheta)$ and $\Delta=[\mathbf{H}_J^+\bm{g}_{\bm{w}} ~ \bm{0}] ^\mathsf{T}$ for brevity of notation.
The local strong-convexity assumption means that $\nabla^2\mathcal L_{\bm 1-\bm w}(\bm\theta)\succeq \sigma_{\min}(\mathbf H_{\bm 1-\bm w})\mathbf I$ for every $\bm\theta\in\mathcal U$.
Together with stationarity at the retrained optimum, it gives
\begin{align}
    \sigma_{\min}(\mathbf H_{\bm 1-\bm w})\|\bm\theta-\boldthetahat(\bm 1-\bm w)\|_2
    \le
    \|\nabla\mathcal L_{\bm 1-\bm w}(\bm\theta)\|_2,
    \qquad \bm\theta\in\mathcal U.
\end{align}
This follows by integrating the Hessian along the segment from the retrained optimum to $\bm\theta$ and applying Cauchy--Schwarz to the resulting gradient inner product.
The same neighborhood contains the segment $\boldthetahat+t\Delta$ for $t\in[0,1]$, on which the Hessian Lipschitz condition is used below.
Because $\mathbf H_{\bm 1-\bm w}$ is evaluated at the reference model, we explicitly assume that its smallest singular value also lower-bounds curvature throughout $\mathcal U$.
With these settings, the error is upper-bounded as
\begin{align}
    \big|\mathrm{Err}_{J}(\bm{w}\mid\boldthetahat)\big|
    &= \big|f\big(\boldthetahat(\bm{1}-\bm{w})\big)-f\big(\boldthetahat^{\mathsf{GIF}}(\bm{1}-\bm{w})\big)\big|
    \le
    C_f
    \|\boldthetahat(\bm{1}-\bm{w})-\boldthetahat^{\mathsf{GIF}}(\bm{1}-\bm{w})\|_2 
    \\
    &\le
    \frac{C_f}{\sigma_{\min}(\mathbf H_{\bm 1-\bm w})}
    \Big\|
    \grad \mathcal{L}_{\bm{1}-\bm{w}}
    \big(\boldthetahat^{\mathsf{GIF}}(\bm{1}-\bm{w})
    \big)
    -
    \hspace{-3mm}
    \cunderbrace{
        \grad \mathcal{L}_{\bm{1}-\bm{w}}
        \big(\boldthetahat(\bm{1}-\bm{w})
        \big)
    }{=0 \text{  (from the first-order optimality)}}
    \hspace{-3mm}
    \Big\|_2
    \\
    &=
    \frac{C_f}{\sigma_{\min}(\mathbf H_{\bm 1-\bm w})}
    \Big\|
    \grad \mathcal{L}_{\bm{1}-\bm{w}}
    \big(
        \cunderbrace{
        \boldthetahat^{\mathsf{GIF}}(\bm{1}-\bm{w})
        }{=\boldthetahat+\Delta}
    \big)
    - 
    \cunderbrace{
        \left[
        \grad\mathcal L_{\bm 1-\bm w}(\boldthetahat)
        +\bm g_{\bm w}
        \right]
    }{=\grad\mathcal L_{\bm 1}(\boldthetahat)=\bm 0}
    \Big\|_2
    \\
    &=
    \frac{C_f}{\sigma_{\min}(\mathbf H_{\bm 1-\bm w})}
    \Big\|
    \cunderbrace{
        \grad \mathcal{L}_{\bm{1}-\bm{w}}
        \big(
            \boldthetahat 
            + 
            \Delta
        \big)
        - 
        \grad \mathcal{L}_{\bm{1}-\bm{w}}
        \big(
            \boldthetahat
        \big)
    }{
        =\int_0^1\grad^2\mathcal{L}_{\bm{1}-\bm{w}}(\boldthetahat+t\Delta)\Delta \mathrm{d}t
    }
    -
    \hspace{-1mm}
    \cunderbrace{
        \mathbf{H}_J\mathbf{H}_J^{+}
        \bm g_{\bm w}
    }{=\mathbf H_{\bm 1-\bm w}\Delta}
    \hspace{-1mm}
    +\bm\Pi_J\bm g_{\bm w}-\bm g_{\bm w}
    \Big\|_2
    \\
    &=
    \frac{C_f}{\sigma_{\min}(\mathbf H_{\bm 1-\bm w})}
    \Big\|
    \int_0^1
    \cunderbrace{
    \big[
        \nabla^2 \mathcal{L}_{\bm{1}-\bm{w}}(\boldthetahat+t\Delta)
        -
        \nabla^2 \mathcal{L}_{\bm{1}-\bm{w}}(\boldthetahat)
    \big]
    }{\text{operator-norm bounded by } C_\mathbf{H}\|\Delta\|_2t}
    \Delta
    \mathrm{d}t
    +\bm\Pi_J\bm g_{\bm w}-\bm g_{\bm w}
    \Big\|_2
    \\
    &\le
    \frac{C_f}{\sigma_{\min}(\mathbf H_{\bm 1-\bm w})}
    \left[
    \|\Delta\|_2
    \int_0^1
    \cunderbrace{
    \big\|
        \nabla^2 \mathcal{L}_{\bm{1}-\bm{w}}(\boldthetahat+t\Delta)
        -
        \nabla^2 \mathcal{L}_{\bm{1}-\bm{w}}(\boldthetahat)
    \big\|_2
    }{\le C_\mathbf{H}\|\Delta\|_2t \text{  (from the Lipschitz condition of } \mathbf{H})}
    \mathrm{d}t
    +\rho(J)
    \right]
    \\
    &\le
    \frac{C_f}{\sigma_{\min}(\mathbf H_{\bm 1-\bm w})}
    \left[
        \frac{C_\mathbf H}{2}\|\Delta\|_2^2+\rho(J)
    \right]
    =
    \frac{C_f}{\sigma_{\min}(\mathbf H_{\bm 1-\bm w})}
    \left[
        \frac{C_\mathbf H}{2}\|\mathbf H_J^+\bm g_{\bm w}\|_2^2+\rho(J)
    \right]
    \\
    &
    \le
    \frac{C_f}{\sigma_{\min}(\mathbf H_{\bm 1-\bm w})}
    \left[
        \frac{C_\mathbf H}{2}
        \frac{\|\bm g_{\bm w}\|_2^2}{\sigma_{\min}(\mathbf H_J)^2}
        +\rho(J)
    \right].
\end{align}
The last inequality uses $\|\mathbf H_J^+\bm g_{\bm w}\|_2\le\|\bm g_{\bm w}\|_2/\sigma_{\min}(\mathbf H_J)$.
When all parameters are editable, $\bm\Pi_J=\mathbf I$ and $\rho(J)=0$, recovering the full-parameter bound in \eqref{eq:full_if_error_bound} under the same local assumptions.

\section{Empirical Diagnostics for the Error Bound}
\label{appendix:error_bound_diagnostics}

We empirically evaluate the comparison condition in \eqref{eq:gif_if_bound_comparison} on the MNIST model used for parameter selection. All selectors retain 10\% of the parameters and values are averaged over 20 seeds. We use the empirical local estimate $\widehat C_{\mathbf H}=130.47$, obtained from ten random perturbation directions of radius $0.01$ with 15 power iterations. For clarity, define the common-factor-free bound gap
\begin{align}
    G(J)
    &:=
    \frac{\widehat C_{\mathbf H}}{2}
    \|\mathbf H^{-1}\bm g_{\bm w}\|_2^2
    -
    \frac{\widehat C_{\mathbf H}}{2}
    \|\mathbf H_J^+\bm g_{\bm w}\|_2^2
    -
    \rho(J).
\end{align}
Positive $G(J)$ means that the measured restricted bound term is smaller than the full-IF term. Both solves use conjugate gradient with regularization $10^{-4}$ and the same convergence criterion.

\begin{table}[htb]
\centering
\caption{MNIST error-bound diagnostics and retained accuracy. Curvature terms and $\rho(J)$ omit the common factor $C_f/\sigma_{\min}(\mathbf H_{\bm 1-\bm w})$.}
\label{tab:mnist_bound_diagnostics}
\small
\begin{tabular}{lrrrrr}
\toprule
Selector & Full IF & Restricted & $\rho(J)$ & $G(J)$ & Retain (\%) \\
\midrule
CAPS & 1.09484 & 1.08018 & $4.84\cdot10^{-3}$ & $9.82\cdot10^{-3}$ & 94.95 \\
Top gradients & 1.09484 & 1.08393 & $4.86\cdot10^{-3}$ & $6.05\cdot10^{-3}$ & 94.85 \\
Top outputs & 1.09484 & 1.08144 & $4.87\cdot10^{-3}$ & $8.52\cdot10^{-3}$ & 94.97 \\
Random & 1.09484 & 1.20962 & $4.69\cdot10^{-3}$ & $-1.19\cdot10^{-1}$ & 87.85 \\
Bottom outputs & 1.09484 & 1.14978 & $5.06\cdot10^{-3}$ & $-6.00\cdot10^{-2}$ & 57.18 \\
Bottom gradients & 1.09484 & 1.25156 & $5.34\cdot10^{-3}$ & $-1.62\cdot10^{-1}$ & 60.83 \\
Reverse CAPS & 1.09484 & 1.31670 & $5.50\cdot10^{-3}$ & $-2.27\cdot10^{-1}$ & 57.10 \\
\bottomrule
\end{tabular}
\end{table}

The well-selected subsets have $G(J)>0$ and retained accuracy above the 93.38\% classical IF baseline. Random and deliberately poor selectors have $G(J)<0$ and lower retained accuracy, as shown in the table. This is an empirical consistency check and does not verify the local assumptions.

\begin{table}[htb]
\centering
\caption{Target-gradient capture and retraining fidelity for the GIF selector. Gap is Retrain minus GIF retain accuracy in percentage points.}
\label{tab:gradient_capture_diagnostics}
\small
\begin{tabular}{lrrrrr}
\toprule
Dataset & Retrain & GIF & Gap & Capture & Normalized residual \\
\midrule
MNIST & 98.09 & 94.95 & 3.14 & 66.45\% & 0.7473 \\
CIFAR-10 & 85.86 & 85.87 & -0.01 & 95.49\% & 0.2968 \\
SVHN & 95.51 & 95.14 & 0.37 & 84.98\% & 0.5271 \\
20 Newsgroups & 65.52 & 65.45 & 0.07 & 98.53\% & 0.1707 \\
PubMed RCT-20k & 87.64 & 87.20 & 0.44 & 99.42\% & 0.1073 \\
\bottomrule
\end{tabular}
\end{table}

Here capture denotes $\|\bm\Pi_J\bm g_{\bm w}\|_2/\|\bm g_{\bm w}\|_2$ and normalized residual denotes $\rho(J)/\|\bm g_{\bm w}\|_2$. MNIST has the lowest capture and largest normalized residual, as well as the largest retain-accuracy gap from retraining. This supports the conditional interpretation that a selected subspace must capture the target-gradient direction as well as maintain stable restricted curvature.

\section{Proof of Theorem~\ref{theorem:Self-loss bound of GIFs}}
\label{appendix:proof_self_loss_bound}

\begin{tcolorbox}[
width=1.00\linewidth, colframe=box_margin, colback=box_background, boxrule=0.3mm, boxsep=-3pt, left=8pt, right=8pt]
    \textbf{Theorem~\ref{theorem:Self-loss bound of GIFs}} (Response of the self-loss)\textbf{.}
    Suppose that $\mathcal L_{\bm w}$ is locally convex and let $\bm{\Pi}_J=\mathbf{H}_J\mathbf{H}_J^+$.
    If $\bm g_{\bm w}$ and $\Delta$ are positively aligned, i.e.,
    $
        \bm g_{\bm w}^{\mathsf T}\Delta\ge0,
    $
    then, for some $c\in[0,1]$, we have
    \begin{align}
        \Delta \mathcal L_J(\bm w\mid\boldthetahat)
        \ge
        \left(
            1 - c
            +
            \frac{\sigma_{\min}(\mathbf{H} _{\bm w})}
            {2\sigma_{\max}(\mathbf H_J)^2}
        \right)
        \|\bm \Pi_J\bm g_{\bm w}\|_2^2 .
    \end{align}
\end{tcolorbox}

As in \textbf{Theorem~\ref{theorem:Error bound of the GIF}},
the masked empirical-risk response must consider that
$\mathbf H_J\in\mathbb R^{n\times k}$ $(n>k)$ cannot span the whole space $\mathbb R^n$ and thus leaves a component $(\mathbf I-\bm\Pi_J)\bm g_{\bm w}$ of the target gradient unrepresented.
Recall that the GIF displacement is
$
    \Delta
    =
    \begin{bmatrix}
        \mathbf H_J^+\bm g_{\bm w} \mid
        \bm 0
    \end{bmatrix}^\mathsf{T}.
$

By Taylor's theorem applied to $\mathcal L_{\bm w}$ along the segment
$\boldthetahat+t\Delta$, $t\in[0,1]$, and by the local convexity lower bound,
\begin{align}
    \Delta \mathcal L_J(\bm w\mid\boldthetahat)
    &=
    \mathcal L_{\bm w}(\boldthetahat+\Delta)
    -
    \mathcal L_{\bm w}(\boldthetahat) \\
    &\ge
    \bm g_{\bm w}^{\mathsf T}\Delta
    +
    \frac{\sigma_{\min}(\mathbf H_{\bm w})}{2}
    \|\Delta\|_2^2 .
\end{align}
Since $\bm g_{\bm w}^{\mathsf T}\Delta\ge0$, there exists $c\in[0,1]$ such that
\begin{align}
    \bm g_{\bm w}^{\mathsf T}\Delta
    \ge
    (1-c)\|\bm\Pi_J\bm g_{\bm w}\|_2^2 .
\end{align}
In the weakest nonnegative-alignment case, this is satisfied by taking $c=1$.

Moreover, since
$    \bm\Pi_J\bm g_{\bm w}
    =
    \mathbf H_J\mathbf H_J^+\bm g_{\bm w},$
we have
\begin{align}
    \|\bm\Pi_J\bm g_{\bm w}\|_2
    &=
    \|\mathbf H_J\mathbf H_J^+\bm g_{\bm w}\|_2 \\
    &\le
    \sigma_{\max}(\mathbf H_J)
    \|\mathbf H_J^+\bm g_{\bm w}\|_2 \\
    &=
    \sigma_{\max}(\mathbf H_J)\|\Delta\|_2 .
\end{align}
Therefore,
\begin{align}
    \|\Delta\|_2^2
    \ge
    \frac{\|\bm\Pi_J\bm g_{\bm w}\|_2^2}
    {\sigma_{\max}(\mathbf H_J)^2}.
\end{align}
Substituting these two lower bounds gives
\begin{align}
    \Delta \mathcal L_J(\bm w\mid\boldthetahat)
    &\ge
    (1-c)\|\bm\Pi_J\bm g_{\bm w}\|_2^2
    +
    \frac{\sigma_{\min}(\mathbf H_{\bm w})}{2}
    \frac{\|\bm\Pi_J\bm g_{\bm w}\|_2^2}
    {\sigma_{\max}(\mathbf H_J)^2} \\
    &=
    \left(
        1-c
        +
        \frac{\sigma_{\min}(\mathbf H_{\bm w})}
        {2\sigma_{\max}(\mathbf H_J)^2}
    \right)
    \|\bm\Pi_J\bm g_{\bm w}\|_2^2 .
\end{align}

\section{Proof of Theorem~\ref{theorem:Convergence of p-LiSSA}}
\label{supple:Proof of Theorem GIF approximation}
\begin{tcolorbox}[
width=1.00\linewidth, colframe=box_margin, colback=box_background, boxrule=0.3mm, boxsep=-3pt, left=8pt, right=8pt]
\textbf{Theorem~\ref{theorem:Convergence of p-LiSSA}} (Convergence of p-LiSSA)\textbf{.}
Assume that $\mathbf{H}_J$ has full column rank. \\
\textbf{\emph{i)}} The p-LiSSA is scale-invariant, i.e., scalar multiplication on $\ell(\cdot)$ does not change $\mathcal{I}\big(\bm{w}, \boldthetahat_J|\boldthetahat, \epsilon\big)$. \\
\textbf{\emph{ii)}} If
$
0<\mu<\frac{2}{\lambda_{\max}(\mathbf{H}_J^\mathsf{T}\mathbf{H}_J)},
$
then \eqref{eq:approximate_series_GIF} converges to $\mathbf{H}_J^{+}\bm{g}_{\bm{w}}$.\\
\textbf{\emph{iii)}} With $\mu = \frac{2}{\lambda_{\max}(\mathbf{H}_J^\mathsf{T}\mathbf{H}_J) + \lambda_{\min}(\mathbf{H}_J^\mathsf{T}\mathbf{H}_J)}$, $\|\mathcal{I}_k - \mathbf{H}_J^+\bm{g}_{\bm{w}}\|_2 \leq \left(\frac{\kappa - 1}{\kappa + 1}\right)^k \|\mathcal{I}_0 - \mathbf{H}_J^+\bm{g}_{\bm{w}}\|_2$, where $\kappa = \frac{\lambda_{\max}(\mathbf{H}_J^\top\mathbf{H}_J)}{\lambda_{\min}(\mathbf{H}_J^\top\mathbf{H}_J)}$.
\end{tcolorbox}

\paragraph{Proof of i).}
We first show the scale-invariance of the GIF.
For any $\mu>0$, scaling $\ell(\cdot)$ and $\mathcal{L}(\cdot)$ by $\mu$ scales both $\mathbf{H}_J$ and $\bm{g}_{\bm{w}}$ by $\mu$.
Hence,
\begin{align}
    \left( (\mu\mathbf{H}_J)^\mathsf{T} (\mu\mathbf{H}_J) \right)^{-1}
    (\mu\mathbf{H}_J)^\mathsf{T} (\mu\bm{g}_{\bm{w}})
    &=
    \left( \mathbf{H}_J^\mathsf{T} \mathbf{H}_J \right)^{-1}
    \mathbf{H}_J^\mathsf{T} \bm{g}_{\bm{w}} \\
    &=
    \mathbf{H}_J^{+}\bm{g}_{\bm{w}}
    =
    \mathcal{I}\big(\bm{w}, \boldthetahat_J|\boldthetahat, \epsilon\big).
\end{align}
Thus, GIF is scale-invariant.

\paragraph{Proof of ii).}
For convergence, since $\mathbf{H}_J$ has full column rank,
$\mathbf{H}_J^\mathsf{T}\mathbf{H}_J$ is symmetric positive definite.
Let $\lambda_i>0$ denote the eigenvalues of $\mathbf{H}_J^\mathsf{T}\mathbf{H}_J$.
Then, the eigenvalues of
$\mathbf{I}-\mu\,\mathbf{H}_J^\mathsf{T}\mathbf{H}_J$
are
$1-\mu\lambda_i$.
Therefore,
\begin{align}
\rho\!\left(\mathbf{I}-\mu\,\mathbf{H}_J^\mathsf{T}\mathbf{H}_J\right)
&<1
\quad\text{whenever}\quad
0<\mu<\frac{2}{\lambda_{\max}(\mathbf{H}_J^\mathsf{T}\mathbf{H}_J)}.
\end{align}
Under this condition, $\Vert(\mathbf{I}-\mu\mathbf{H}_J^\mathsf{T}\mathbf{H})^i\Vert\rightarrow 0$ as $k\rightarrow \infty$, and the Neumann series converges as follows:
\begin{align}
    \lim_{k\rightarrow\infty}\mathcal{I}_k
    =
    -\sum_{i=0}^{\infty} \left(\mathbf{I} - \mu\mathbf{H}_J^{\mathsf{T}}\mathbf{H}_J\right)^i \mu\mathbf{H}_J^{\mathsf{T}}\bm{g}_{\bm{w}}
    =
    \left( \mathbf{H}_J^\mathsf{T}\mathbf{H}_J \right)^{-1}
    \mathbf{H}_J^\mathsf{T}\bm{g}_{\bm{w}}
    =
    \mathbf{H}_J^{+}\bm{g}_{\bm{w}}.
\end{align}

\paragraph{Proof of iii).}
Let $\mathbf{e}_k = \mathcal{I}_k - \mathbf{H}_J^+\bm{g}_{\bm{w}}$.
Since $\mathbf{H}_J$ has full column rank, $\mathbf{H}_J^\top\mathbf{H}_J$ is symmetric positive definite and $\mathbf{H}_J^\top\mathbf{H}_J\mathbf{H}_J^+ = \mathbf{H}_J^\top$.
Then we have
\begin{align}
    \mathbf{e}_k 
    &= \mathcal{I}_k - \mathbf{H}_J^+\bm{g}_{\bm{w}} \nonumber\\
    &\overset{(a)}{=} \mu\mathbf{H}_J^\top\bm{g}_{\bm{w}} + \left(\mathbf{I} - \mu\mathbf{H}_J^\top\mathbf{H}_J\right)\mathcal{I}_{k-1} - \mathbf{H}_J^+\bm{g}_{\bm{w}} \nonumber\\
    &\overset{(b)}{=} \mu\mathbf{H}_J^\top\bm{g}_{\bm{w}} + \left(\mathbf{I} - \mu\mathbf{H}_J^\top\mathbf{H}_J\right)\mathcal{I}_{k-1} - \mu\mathbf{H}_J^\top\bm{g}_{\bm{w}} - \left(\mathbf{I} - \mu\mathbf{H}_J^\top\mathbf{H}_J\right)\mathbf{H}_J^+\bm{g}_{\bm{w}} \\
    &= \left(\mathbf{I} - \mu\mathbf{H}_J^\top\mathbf{H}_J\right)\left(\mathcal{I}_{k-1} - \mathbf{H}_J^+\bm{g}_{\bm{w}}\right)\nonumber \\
    &= \left(\mathbf{I} - \mu\mathbf{H}_J^\top\mathbf{H}_J\right)\mathbf{e}_{k-1},\nonumber
\end{align}
where the equality (a) is obtained from \eqref{eq:approximate_series_GIF}, and the equality (b) holds because $\mathbf{H}_J^\top\mathbf{H}_J\mathbf{H}_J^+ = \mathbf{H}_J^\top$.

Since $\mathbf{I} - \mu\mathbf{H}_J^\top\mathbf{H}_J$ is symmetric, $\|(\mathbf{I} - \mu\mathbf{H}_J^\top\mathbf{H}_J)^k\|_2 = \rho(\mathbf{I} - \mu\mathbf{H}_J^\top\mathbf{H}_J)^k$, and therefore
\begin{align}
    \|\mathbf{e}_k\|_2 
    \leq \left\|\left(\mathbf{I} - \mu\mathbf{H}_J^\top\mathbf{H}_J\right)^k\right\|_2 \|\mathbf{e}_0\|_2
    = \rho\!\left(\mathbf{I} - \mu\mathbf{H}_J^\top\mathbf{H}_J\right)^k \|\mathbf{e}_0\|_2,
\end{align}
where the inequality holds because $\|\mathbf{A}\mathbf{x}\|_2 \leq \|\mathbf{A}\|_2\|\mathbf{x}\|_2$.

Let $\lambda_{\min} = \lambda_{\min}(\mathbf{H}_J^\top\mathbf{H}_J)$ and $\lambda_{\max} = \lambda_{\max}(\mathbf{H}_J^\top\mathbf{H}_J)$.
The eigenvalues of $\mathbf{I} - \mu\mathbf{H}_J^\top\mathbf{H}_J$ are $\{1 - \mu\lambda_i\}$, so
\begin{align}
    \rho\!\left(\mathbf{I} - \mu\mathbf{H}_J^\top\mathbf{H}_J\right) 
    = \max\!\left(|1 - \mu\lambda_{\min}|,\; |1 - \mu\lambda_{\max}|\right).
\end{align}
The eigenvalues of $\mathbf{I} - \mu\mathbf{H}_J^\top\mathbf{H}_J$ are $\{1 - \mu\lambda_i\}$, so
\begin{align}
    \rho\!\left(\mathbf{I} - \mu\mathbf{H}_J^\top\mathbf{H}_J\right) 
    = \max_i|1 - \mu\lambda_i| 
    = \max\!\left(|1 - \mu\lambda_{\min}|,\; |1 - \mu\lambda_{\max}|\right).
\end{align}

This is minimized when $|1 - \mu\lambda_{\min}| = |1 - \mu\lambda_{\max}|$, 
i.e., $1 - \mu\lambda_{\min} = \mu\lambda_{\max} - 1$, yielding
\begin{align}
    \mu^* = \frac{2}{\lambda_{\max} + \lambda_{\min}}, \qquad
    \rho^* = \frac{\lambda_{\max} - \lambda_{\min}}{\lambda_{\max} + \lambda_{\min}} = \frac{\kappa - 1}{\kappa + 1},
\end{align}
where $\kappa = \lambda_{\max} / \lambda_{\min}$.
Substituting $\rho^*$ gives
\begin{align}
    \|\mathbf{e}_k\|_2 \leq \left(\frac{\kappa-1}{\kappa+1}\right)^k \|\mathbf{e}_0\|_2,
\end{align}
where the convergence rate is given by $\frac{\kappa-1}{\kappa+1}$.
\section{Proof of Theorem~\ref{theorem:Unbiasedness of p-LiSSA}}
\label{sec:proof_unbiased_plissa}

In this section, we provide a formal proof that the p-LiSSA algorithm becomes an unbiased estimator of the GIF update vector when utilizing a double sampling strategy.

\paragraph{Stochastic Estimator Definition.}
Recall that the target update vector for GIF is given by the limit of the Neumann series:
\begin{align}
    \Delta\boldthetahat_{\bm{1}+\epsilon\bm{w}} \propto -\mathbf{H}_J^{+}\bm{g}_{\bm{w}} = -\sum_{i=0}^{\infty} \left(\mathbf{I} - \mathbf{H}_J^{\mathsf{T}}\mathbf{H}_J\right)^i \mathbf{H}_J^{\mathsf{T}}\bm{g}_{\bm{w}}.
\end{align}
Let $\tilde{\mathcal{I}}_k$ denote the stochastic estimate at iteration $k$. The standard p-LiSSA recurrence using a single batch is biased because 
\begin{align} 
\mathbb{E}[\mathbf{H}_{J, B}^{\mathsf{T}}\mathbf{H}_{J, B}] =
\mathbb{E}[\mathbf{H}_{J, B}]^{\mathsf{T}}\mathbb{E}[\mathbf{H}_{J, B}] +
\cunderbrace{\mathrm{Cov}(\mathbf{H}_{J, B}, \mathbf{H}_{J, B})}{\neq\mathbf{0}}.
\end{align}

To address this, we introduce the double sampling update rule.

At each iteration $k$, we independently sample two mini-batches $B_{k,1}$ and $B_{k,2}$ to compute the Hessian sub-matrices $\mathbf{H}_{J, B_{k,1}}$ and $\mathbf{H}_{J, B_{k,2}}$. The stochastic recurrence is defined as:
\begin{align}
    \tilde{\mathcal{I}}_0 &= \mathbf{H}_{J, B_0}^{\mathsf{T}} \bm{g}_{\bm{w}}, \\
    \tilde{\mathcal{I}}_k &= \tilde{\mathcal{I}}_0 + \left(\mathbf{I} - \mathbf{H}_{J, B_{k,1}}^{\mathsf{T}}\mathbf{H}_{J, B_{k,2}}\right)\tilde{\mathcal{I}}_{k-1},
\end{align}
where $\tilde{\mathcal{I}}_0$ is also computed using an independent batch $B_0$.

\begin{tcolorbox}[
width=1.00\linewidth, colframe=box_margin, colback=box_background, boxrule=0.3mm, boxsep=-3pt, left=8pt, right=8pt]
\textbf{Theorem~\ref{theorem:Unbiasedness of p-LiSSA}} (Unbiasedness of p-LiSSA)\textbf{.}
    Assume that the batches $B_{k,1}, B_{k,2}$ are drawn independently to compute $\mathbf{H}_{J, {k,1}}$ and $\mathbf{H}_{J, {k,2}}$.
    Define the stochastic recurrence as
    \begin{align}
        \tilde{\mathcal{I}}_k = \tilde{\mathcal{I}}_0\! + \!\left(\mathbf{I} - \mathbf{H}_{J, {k,1}}^{\mathsf{T}}\mathbf{H}_{J, {k,2}}\right)\!\tilde{\mathcal{I}}_{k-1},~
        \tilde{\mathcal{I}}_0 = \mathbf{H}_{J, {0,1}}^{\mathsf{T}} \bm{g}_{\bm{w}}
        .
        \label{eq:double_sampling_recurrence}
    \end{align}
    Then, for any $k \ge 0$, the expected value of $\tilde{\mathcal{I}}_k$ satisfies
    \begin{align}
        \textstyle
        \mathbb{E}[\tilde{\mathcal{I}}_k] = \sum_{j=0}^{k} \left(\mathbf{I} - \mathbf{H}_J^{\mathsf{T}}\mathbf{H}_J\right)^j \mathbf{H}_J^{\mathsf{T}}\bm{g}_{\bm{w}}.
    \end{align}
    Consequently, as $k \to \infty$, $\mathbb{E}[\tilde{\mathcal{I}}_k]$ converges to $\mathbf{H}_J^{+}\bm{g}_{\bm{w}}$.
\end{tcolorbox}
We prove the claim by induction on $k$.

\paragraph{Base case ($k=0$).}
By linearity of expectation and the unbiasedness of the mini-batch Hessian,
\begin{align}
\mathbb{E}\big[\tilde{\mathcal{I}}_0\big]
&= \mathbb{E}\big[\mathbf{H}_{J, {0,1}}^{\mathsf{T}}\bm{g}_{\bm{w}}\big]
= \mathbb{E}\big[\mathbf{H}_{J, {0,1}}\big]^{\mathsf{T}}\bm{g}_{\bm{w}}
= \mathbf{H}_J^{\mathsf{T}}\bm{g}_{\bm{w}}.
\end{align}
This coincides with the $j=0$ term of the Neumann expansion.

\paragraph{Inductive step.}
Assume that, for $k-1$, the following identity holds:
\begin{align}
\mathbb{E}\big[\tilde{\mathcal{I}}_{k-1}\big]
= \sum_{j=0}^{k-1}\big(\mathbf{I}-\mathbf{H}_J^{\mathsf{T}}\mathbf{H}_J\big)^j \mathbf{H}_J^{\mathsf{T}}\bm{g}_{\bm{w}}.
\end{align}
Taking expectation on both sides of \eqref{eq:double_sampling_recurrence} yields
\begin{align}
\mathbb{E}\big[\tilde{\mathcal{I}}_k\big]
&= \mathbb{E}\big[\tilde{\mathcal{I}}_0\big]
+ \mathbb{E}\big[\tilde{\mathcal{I}}_{k-1}\big]
- \mathbb{E}\big[\mathbf{H}_{J, {k,1}}^{\mathsf{T}}\mathbf{H}_{J, {k,2}}\tilde{\mathcal{I}}_{k-1}\big].
\end{align}
Since the mini-batches $B_{k,1}$ and $B_{k,2}$ are sampled independently at iteration $k$, and are independent of
$\tilde{\mathcal{I}}_{k-1}$ (which depends only on the batches from earlier iterations), we can factor the expectation as
\begin{align}
\mathbb{E}\big[\mathbf{H}_{J, {k,1}}^{\mathsf{T}}\mathbf{H}_{J, {k,2}}\tilde{\mathcal{I}}_{k-1}\big]
= \mathbb{E}\big[\mathbf{H}_{J, {k,1}}^{\mathsf{T}}\big]
   \mathbb{E}\big[\mathbf{H}_{J, {k,2}}\big]
   \mathbb{E}\big[\tilde{\mathcal{I}}_{k-1}\big]
= \mathbf{H}_J^{\mathsf{T}}\mathbf{H}_J \, \mathbb{E}\big[\tilde{\mathcal{I}}_{k-1}\big].
\end{align}
Therefore,
\begin{align}
\mathbb{E}\big[\tilde{\mathcal{I}}_k\big]
&= \mathbf{H}_J^{\mathsf{T}}\bm{g}_{\bm{w}}
+ \big(\mathbf{I}-\mathbf{H}_J^{\mathsf{T}}\mathbf{H}_J\big)\mathbb{E}\big[\tilde{\mathcal{I}}_{k-1}\big].
\end{align}
Applying the inductive hypothesis gives
\begin{align}
\mathbb{E}\big[\tilde{\mathcal{I}}_k\big]
&= \mathbf{H}_J^{\mathsf{T}}\bm{g}_{\bm{w}}
+ \big(\mathbf{I}-\mathbf{H}_J^{\mathsf{T}}\mathbf{H}_J\big)
   \sum_{j=0}^{k-1}\big(\mathbf{I}-\mathbf{H}_J^{\mathsf{T}}\mathbf{H}_J\big)^j \mathbf{H}_J^{\mathsf{T}}\bm{g}_{\bm{w}} \\
&= \sum_{j=0}^{k}\big(\mathbf{I}-\mathbf{H}_J^{\mathsf{T}}\mathbf{H}_J\big)^j \mathbf{H}_J^{\mathsf{T}}\bm{g}_{\bm{w}}.
\end{align}
This completes the induction. Hence, p-LiSSA with double sampling provides an unbiased estimate of the truncated Neumann
series for the GIF.

\section{A Blockwise Interpretation of CAPS}
\label{appendix:blockwise_caps}

This section provides a derivation that motivates the blockwise CAPS score.
The derivation should be understood as an interpretation of the restricted least-squares structure of GIF.
Our goal is to explain why a block score should jointly depend on a target-gradient term and a curvature term, and why the practical score in Sec.~5 preserves this structure.

Recall that GIF is defined through the restricted linearized problem
\begin{align}
    \hat{\theta}_J(1+\epsilon \bm w)
    :=
    \argmin_{\theta\in\mathbb R^k}
    \left\|
    \epsilon \bm g_{\bm w}
    +
    \mathbf H_J(\theta-\hat{\theta}_J)
    \right\|_2^2,
\end{align}
whose solution is
\begin{align}
    \Delta\hat{\theta}_{J,1+\epsilon\bm w}
    =
    -\epsilon \mathbf H_J^+\bm g_{\bm w}.
\end{align}
Thus, GIF is a restricted update that minimizes the residual of the linearized first-order condition on the editable subspace.

To assess the quality of a candidate editable block $B$, consider the same restricted problem when only the parameters in $B$ are allowed to move.
Let $\Delta_B := \theta_B-\hat{\theta}_B$.
Then the corresponding blockwise linearized problem is
\begin{align}
    \min_{\Delta_B}
    \left\|
    \epsilon \bm g_{\bm w}
    +
    \mathbf H_B \Delta_B
    \right\|_2^2,
\end{align}
where $\mathbf H_B$ denotes the columns of the Hessian indexed by block $B$.
Expanding the squared norm gives
\begin{align}
    \left\|
    \epsilon \bm g_{\bm w}
    +
    \mathbf H_B \Delta_B
    \right\|_2^2
    &=
    \|\epsilon \bm g_{\bm w}\|_2^2
    +
    2\epsilon\,\bm g_{\bm w}^{\mathsf T}\mathbf H_B\Delta_B
    +
    \Delta_B^{\mathsf T}\mathbf H_B^{\mathsf T}\mathbf H_B\Delta_B.
\end{align}
The constant term does not affect the ranking of candidate blocks.
Block quality therefore depends on both the linear term $\mathbf H_B^{\mathsf T}\bm g_{\bm w}$ and the quadratic term $\mathbf H_B^{\mathsf T}\mathbf H_B$. The former measures the response of block $B$ to the target gradient, while the latter describes the local geometry of the restricted solve.

Assume that $\mathbf H_B$ has full column rank.
Then the unique minimizer is
\begin{align}
    \Delta_B^\star
    =
    -\epsilon(\mathbf H_B^{\mathsf T}\mathbf H_B)^{-1}\mathbf H_B^{\mathsf T}\bm g_{\bm w}
    =
    -\epsilon \mathbf H_B^+\bm g_{\bm w}.
\end{align}
Substituting this solution back into the objective yields
\begin{align}
    \min_{\Delta_B}
    \left\|
    \epsilon \bm g_{\bm w}
    +
    \mathbf H_B \Delta_B
    \right\|_2^2
    =
    \epsilon^2\|\bm g_{\bm w}\|_2^2
    -
    \epsilon^2
    (\mathbf H_B^{\mathsf T}\bm g_{\bm w})^{\mathsf T}
    (\mathbf H_B^{\mathsf T}\mathbf H_B)^{-1}
    (\mathbf H_B^{\mathsf T}\bm g_{\bm w}).
\end{align}
Therefore, up to the constant factor $\epsilon^2$, the predicted reduction achieved by block $B$ is
\begin{align}
    (\mathbf H_B^{\mathsf T}\bm g_{\bm w})^{\mathsf T}
    (\mathbf H_B^{\mathsf T}\mathbf H_B)^{-1}
    (\mathbf H_B^{\mathsf T}\bm g_{\bm w}).
    \label{eq:exact_block_gain}
\end{align}
For unlearning, where $\epsilon=-1$, the factor $\epsilon^2$ is fixed and can be ignored for ranking.
Equation~\eqref{eq:exact_block_gain} shows that a favorable block should both i) have a strong response to the target gradient and ii) admit a stable restricted solve through a well-behaved Gram matrix $\mathbf H_B^{\mathsf T}\mathbf H_B$.
This is exactly the same gradient--curvature structure that appears in the GIF solution and in the p-LiSSA iteration through $\mathbf H_J^{\mathsf T}\mathbf H_J$.

The exact quantity in Eq.~\eqref{eq:exact_block_gain} is expensive to evaluate for every candidate block, as it requires explicitly forming or repeatedly applying $\mathbf H_B$ and $\mathbf H_B^{\mathsf T}\mathbf H_B$.
For this reason, Sec.~5 adopts a practical proxy that preserves the same structure while replacing the exact blockwise least-squares terms by computable quantities.
Specifically, we replace
\begin{align}
    \mathbf H_B^{\mathsf T}\bm g_{\bm w}
    \quad \text{by} \quad
    \bm g_{\bm w,B},
    \qquad
    \mathbf H_B^{\mathsf T}\mathbf H_B
    \quad \text{by} \quad
    \mathbf C_B+\lambda \mathbf I,
\end{align}
where $\bm g_{\bm w,B}$ is the target-data gradient restricted to block $B$, $\mathbf C_B$ is the local curvature proxy defined in Sec.~5, and $\lambda \ge 0$ is a small numerical stabilizer.
This leads to the blockwise CAPS score
\begin{align}
    S(B)
    :=
    \bm g_{\bm w,B}^{\mathsf T}
    (\mathbf C_B+\lambda \mathbf I)^{-1}
    \bm g_{\bm w,B}.
\end{align}
Accordingly, $S(B)$ should be interpreted as a practical blockwise proxy for the gain suggested by the restricted GIF solve, rather than as an exact identity derived from Eq.~(15).

This interpretation also clarifies the role of the two components of CAPS.
The gradient term favors blocks that are strongly associated with the target data, while the inverse-curvature term suppresses blocks whose local geometry is too flat or ill-conditioned to support a stable restricted update.
Consistent with \textbf{Theorem~\ref{theorem:Error bound of the GIF}}, we therefore discard numerically unstable blocks before ranking them by $S(B)$.
This yields a block selection rule that is aligned with the restricted least-squares structure of GIFs while remaining computationally tractable.
\section{Reproducibility and Compute Resources}
\label{appendix:reproducibility}

\paragraph{Dataset.}
We use five public benchmark datasets.
For MNIST~\cite{Lecun1998-MNIST}, we use the Hugging Face
\texttt{ylecun/mnist} dataset entry, which lists the MIT License.
For CIFAR-10~\cite{Krizhevsky2009-cifar10}, we use the standard CIFAR-10
benchmark. The original distribution page does not state an explicit dataset
license, and the UCI entry directs users there for licensing information.\footnote{\url{https://www.cs.toronto.edu/~kriz/cifar.html} and \url{https://archive.ics.uci.edu/dataset/691/cifar+10}}
For SVHN~\cite{Netzer2011-svhn}, we use the Street View House Numbers
benchmark from its public research distribution, which is marked for
non-commercial use only.\footnote{\url{http://ufldl.stanford.edu/housenumbers/}}
For 20 Newsgroups~\cite{Lang1995-newsweeder}, we use the standard
20 Newsgroups text-classification benchmark. The original distribution page
does not specify a corpus-wide license.\footnote{\url{http://qwone.com/~jason/20Newsgroups/}}
For PubMed 20k RCT~\cite{Dernoncourt2017-pubmed-rct}, we use the public
PubMed RCT benchmark. The dataset authors refer users to NLM's PubMed terms,
which note that some abstracts may be protected by copyright.\footnote{\url{https://github.com/Franck-Dernoncourt/pubmed-rct} and \url{https://ftp.ncbi.nlm.nih.gov/pubmed/baseline/README.txt}}
For text classification, we initialize from the Hugging Face checkpoint
\texttt{google/bert\_uncased\_L-2\_H-128\_A-2}, which is associated with
BERT~\cite{devlin2019-bert} and lists the Apache-2.0 license.

\paragraph{Model training.}
The five base checkpoints used in the model-editing experiments were trained
under a common supervised-learning protocol.
All checkpoints were optimized for 100 epochs with cross-entropy loss, cosine
annealing of the learning rate, and model selection by validation accuracy.
Epoch-wise checkpoints were also saved to provide the training trajectories used
by TracIn~\cite{Pruthi2020-TracIn}.
For class-removal experiments, the corresponding retrained reference models were
trained with the same procedure after excluding the target class from each
training batch.

For MNIST~\cite{Lecun1998-MNIST}, we trained an eight-layer fully connected
network with hidden width 512 and dropout probability 0.1 using SGD with
learning rate 0.05, momentum 0.9, weight decay $5\times10^{-4}$, and batch size
256.
For CIFAR-10~\cite{Krizhevsky2009-cifar10} and
SVHN~\cite{Netzer2011-svhn}, we trained ResNet-18 and VGG-11 checkpoints,
respectively, using the same SGD hyperparameters and batch size.
The CIFAR-10 and SVHN models use the Hugging Face dataset pipeline, which
reserves a validation split from the training data when no official validation
split is provided.
For 20 Newsgroups~\cite{Lang1995-newsweeder} and PubMed 20k
RCT~\cite{Dernoncourt2017-pubmed-rct}, we fine-tuned
\texttt{google/bert\_uncased\_L-2\_H-128\_A-2} with a linear classification
head.
The text models were trained with AdamW, learning rate $2\times10^{-5}$, weight
decay $10^{-2}$, batch size 32, maximum sequence length 80, dropout probability
0.1, and gradient clipping at norm 1.0.
After 100 training epochs, the selected checkpoints achieved held-out test
accuracies of 98.30\% on MNIST, 86.09\% on CIFAR-10, 95.24\% on SVHN, 65.84\%
on 20 Newsgroups, and 83.96\% on PubMed 20k RCT.
All numerical experiments were run on a local Linux workstation.
The workstation contains an AMD Ryzen Threadripper 7970X processor with 32
physical cores and 64 hardware threads, 256 GB of system memory, and two
NVIDIA RTX 6000 Ada Generation GPUs with 48 GB of memory per GPU.
The NVIDIA driver version is 580.126.09, and CUDA version is 13.0.

\paragraph{Model edit.}
Each empirical study was repeated across independent random seeds, and all
aggregate results in Sec.~\ref{sec:Numerical Experiments} were computed from the
corresponding per-seed measurements.
For the solver study, we additionally recorded elapsed wall-clock time and peak
memory for each run.
These measurements are reported together with the residual metric in
Fig.~\ref{fig:solver_mnist}.
We use class-level data removal as the model-editing task.
In the theoretical formulation of Eq.~\eqref{eq:upweighted empirical risk},
setting $\epsilon=-1$ gives
\begin{align}
    \frac{1}{m}\sum_{i=1}^m(1-w_i)\ell(z_i,\boldtheta),
\end{align}
which removes the target examples from the original empirical risk.
The retrained reference model, however, minimizes the same retained loss after
renormalizing by the number of retained examples,
\begin{align}
    \frac{1}{m-\|\bm w\|_1}
    \sum_{i=1}^m(1-w_i)\ell(z_i,\boldtheta).
\end{align}
This normalization matters in implementation because the target loss is
computed as an average over the selected target examples, while the retained
loss is computed as an average over the retained examples.
Consequently, the reduced-risk first-order condition around $\boldthetahat$ is
linearized as
\begin{align}
    0
    &\approx
    -\frac{1}{m-\|\bm w\|_1}
    \sum_{i=1}^m w_i\nabla\ell(z_i,\boldthetahat)
    +
    \frac{1}{m-\|\bm w\|_1}
    \sum_{i=1}^m(1-w_i)\nabla^2\ell(z_i,\boldthetahat)\Delta
    \\
    &=
    -\frac{m}{m-\|\bm w\|_1}\bm g_{\bm w}
    +
    \frac{m}{m-\|\bm w\|_1}\mathbf H_{\bm 1-\bm w}\Delta.
\end{align}
Multiplying the last line by $(m-\|\bm w\|_1)/m$ recovers the same linearized
system as Eq.~\eqref{eq:taylor_derivation_influence} at $\epsilon=-1$.
Empirically, this corresponds to computing HVPs on retained examples and
multiplying the averaged target loss by
$\|\bm w\|_1/(m-\|\bm w\|_1)$.

The remaining discrepancy is that influence functions provide a local
linearized direction, whereas the experiments require finite model edits.
Following the standard observation that fixed theoretical influence magnitudes
can underestimate finite data-removal effects in neural networks
\citep{Koh2019_GIF}, we use the solver output as an edit direction and choose
the finite displacement empirically.
We normalize the influence direction and apply a sequence of candidate edits
\begin{align}
    \boldtheta_J^{(t+1)}
    =
    \boldtheta_J^{(t)}
    +
    \gamma
    \frac{
        \mathcal I\big(\bm w,\boldthetahat_J\mid\boldthetahat,\epsilon\big)
    }{
        \|\mathcal I\big(\bm w,\boldthetahat_J\mid\boldthetahat,\epsilon\big)\|_2
    },
    \quad
    \boldtheta_J^{(0)}=\boldthetahat_J,
\end{align}
where $\gamma$ is the finite edit scale.
After each candidate step, we evaluate target-class accuracy and retained
accuracy.
The reported edited model is the candidate with the best model-editing score,
with GIF additionally required to satisfy the target-accuracy constraint used in
Table~\ref{tab:model_edit_comparison}.

\section{Do GIF Updates Behave Like From-Scratch Retraining?}
The GIF only updates a small model segment, leaving most parameters unchanged from their original values.
However, this might contradict the nature of influence functions, which require the updated model to behave similarly to the model retrained from scratch with the edited dataset.
We further verify how the updated model internally behaves by observing output distributions and Gradient-weighted Class Activation Mapping (Grad-CAM) visualization \cite{Ramprasaath2016-GradCam} of its outputs.

\begin{figure}[ht]
    \centering
    \begin{minipage}[t]{0.6\linewidth}
    \centering
        \includegraphics[width=\linewidth]{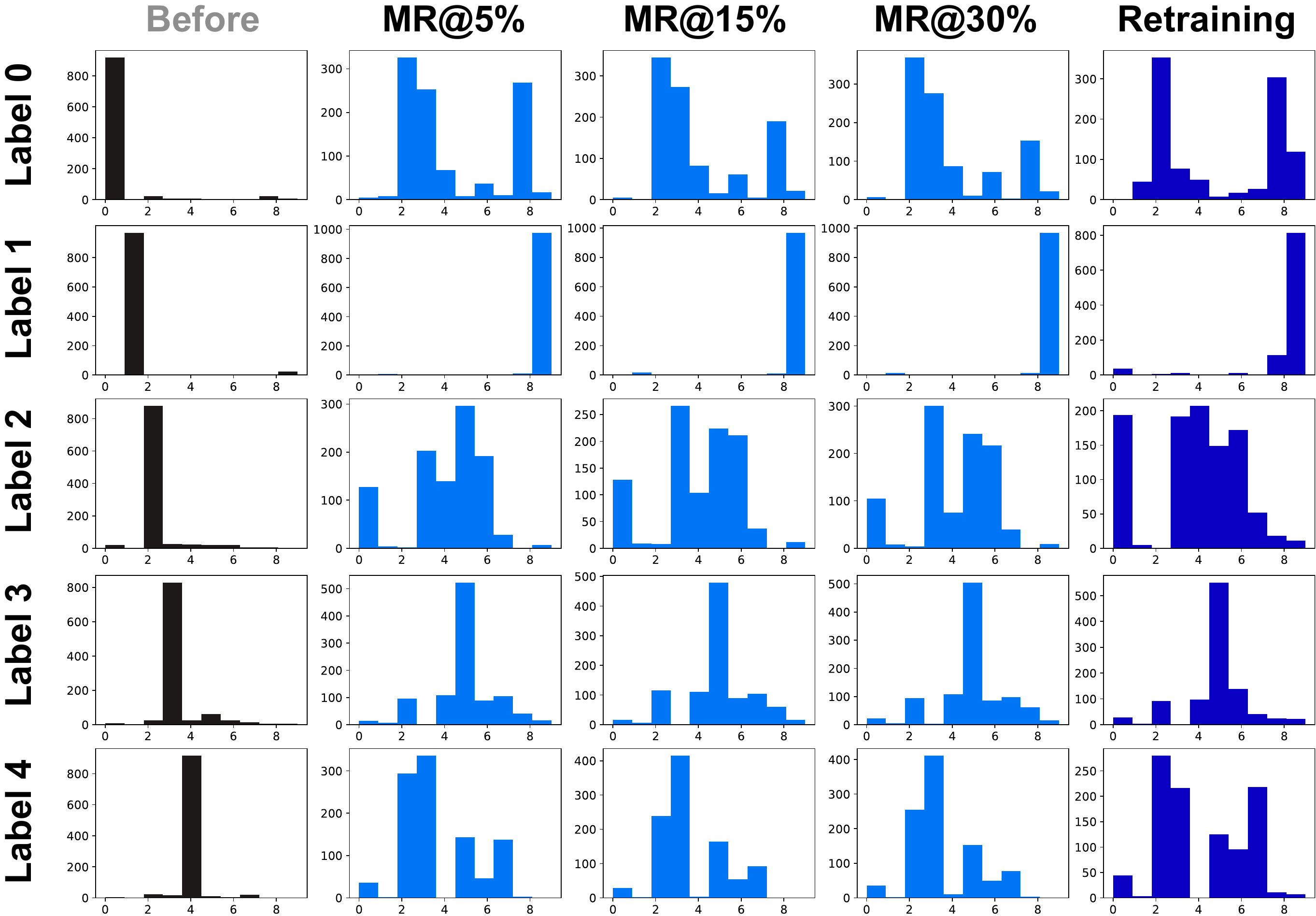}
    \end{minipage}
    \quad 
    \begin{minipage}[t]{0.6\linewidth}
    \centering
        \includegraphics[width=\linewidth]{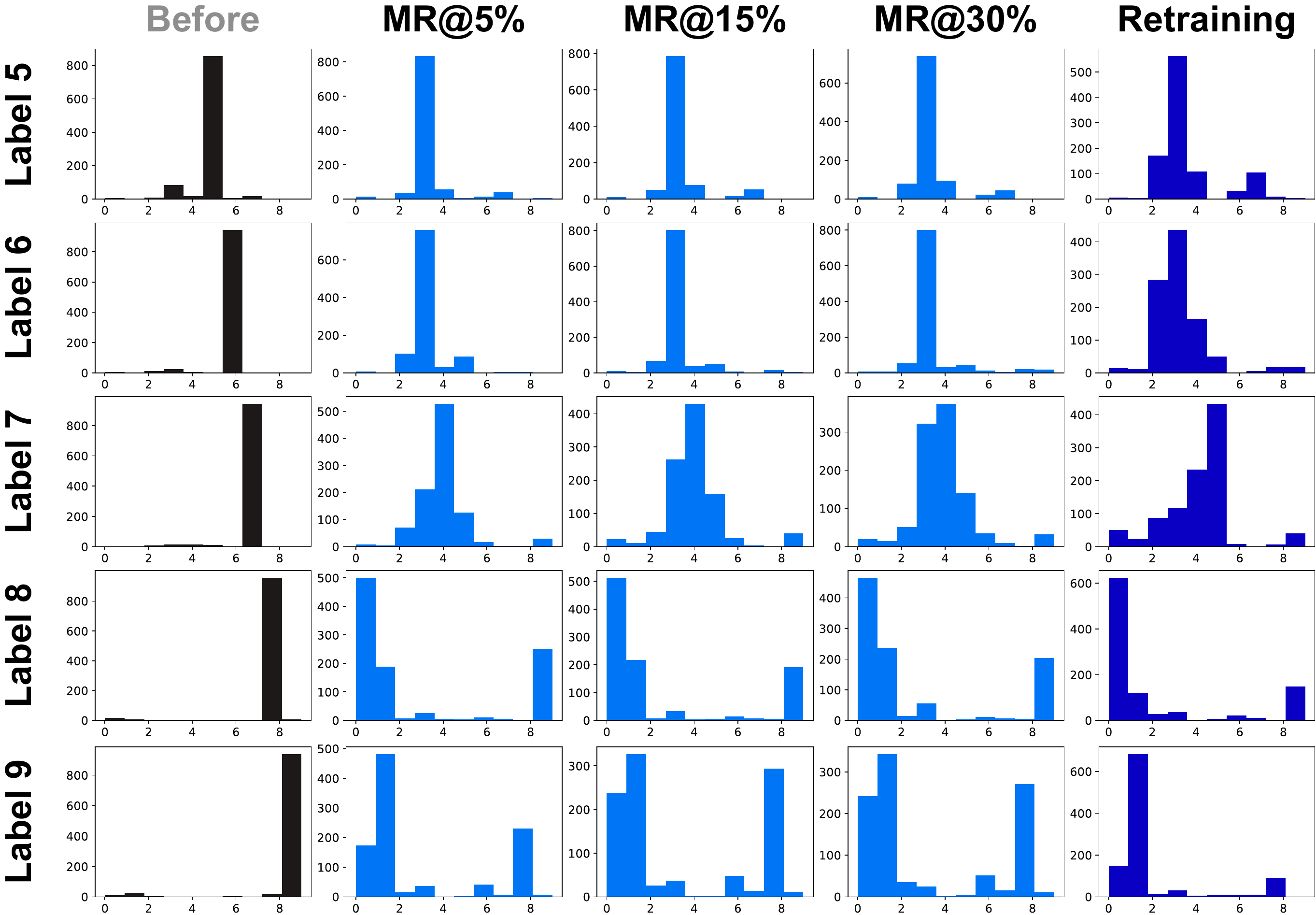}
    \end{minipage}
    \caption{Inference histograms for the original model before removal, the three models updated by GIF, and the retrained model.}
    \label{fig:inference histogram}
\end{figure}

\paragraph{Inference histogram.}
We first train the VGG-11 network (\textbf{Before}) with the CIFAR-10 dataset, and then perform a data removal task using the GIF.
The \textbf{Top-$\rho\%$ gradients} method selects {5, 15, 30}\% of model parameters. The selected parameters are then updated via the GIF, referred to as \textbf{MR@{5, 15, 30}\%}, respectively.
Lastly, we retrain the network from scratch (\textbf{Retraining}) while excluding the target removal data.

Fig.~\ref{fig:inference histogram} shows how often each label is inferred when deleted data is provided to each network.
The histograms obtained from \textbf{MR@{5, 15, 30}\%} show a pattern similar to the \textbf{Retraining}.
The results indicate that the updated networks properly encode features of the removed data, similar to the \textbf{Retraining}, even though only a small portion of parameters are updated.
This contrasts with the approach proposed in \cite{yoon2023fewshot}, which renders features of removed data randomly.

\begin{figure}[ht]
    \centering
    \begin{minipage}[t]{0.6\linewidth}
    \centering
        \includegraphics[width=\linewidth]{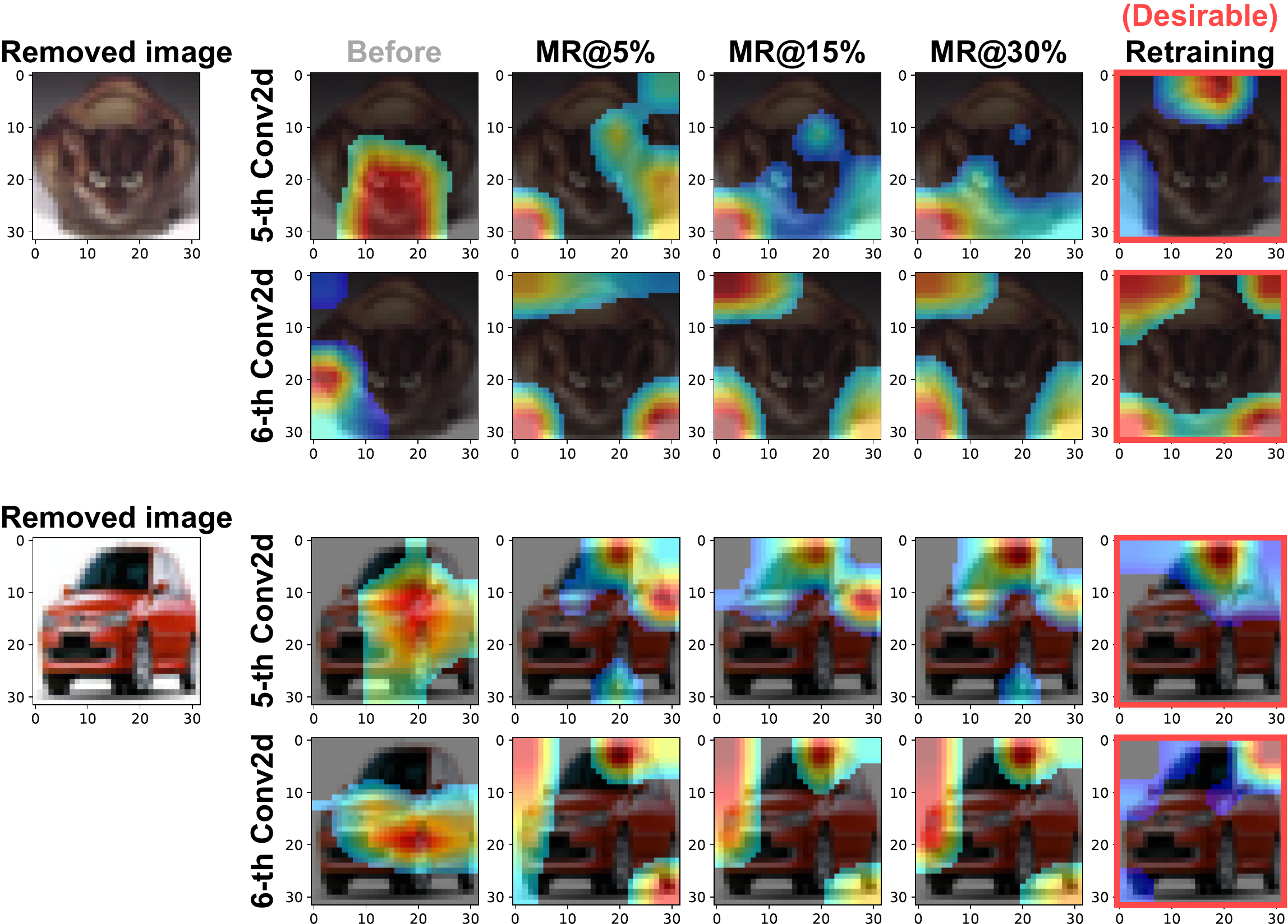}
    \end{minipage}
    \quad 
    \begin{minipage}[t]{0.6\linewidth}
    \centering
        \includegraphics[width=\linewidth]{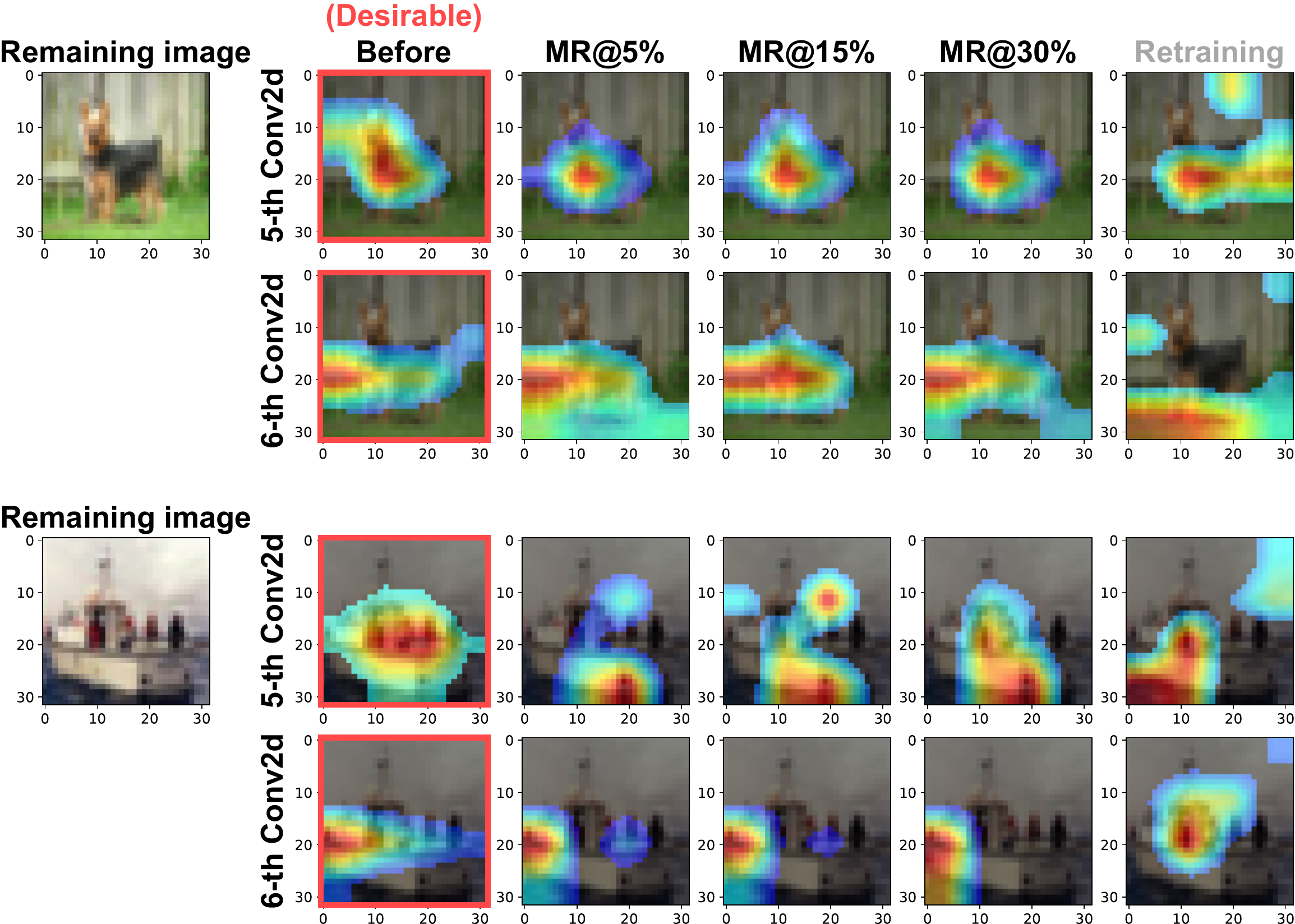}
    \end{minipage}
    \caption{Grad-CAM visualization from the five models for the removed and retained images.}
    \label{fig:grad_cam}
\end{figure}

\paragraph{Grad-CAM visualization.}
We visualize Grad-CAM results in Fig.~\ref{fig:grad_cam} by feeding the removed and remaining images through the five models obtained by the same process as the inference histogram.
The fifth and sixth convolutional layers (Conv2d) are selected for analysis because gradient outputs from earlier layers provide overly detailed outputs, while those from later layers are excessively coarse.
The updated models \textbf{MR@{5, 15, 30}\%} show heatmaps similar to those of \textbf{retraining} for the removed images.
Interestingly, these four models tend to avoid the heatmap obtained from the model \textbf{before}.
For the remaining images, the heatmap patterns from \textbf{MR@{5, 15, 30}\%} are more similar to the \textbf{before} than the \textbf{retraining}, as they mostly share the same parameters.
However, all patterns show a negligible discrepancy since the five models provide correct inferences for the remaining images.


\ifshowchecklist
\newpage
\section*{NeurIPS Paper Checklist}

\begin{enumerate}

\item {\bf Claims}
    \item[] Question: Do the main claims made in the abstract and introduction accurately reflect the paper's contributions and scope?
    \item[] Answer: \answerYes{}
    \item[] Justification: The abstract and introduction state the paper's main contributions on generalized influence functions, target-parameter selection, and p-LiSSA. These claims are supported by the theoretical development in Secs.~\ref{sec:Proposed Method: Generalized IF} and~\ref{sec:param_sel}, and by the numerical experiments in Sec.~\ref{sec:Numerical Experiments}.

\item {\bf Limitations}
    \item[] Question: Does the paper discuss the limitations of the work performed by the authors?
    \item[] Answer: \answerYes{}
    \item[] Justification: Sec.~\ref{sec:Limitations, Challenges, and Future Works} discusses the trade-off between conditioning and target-gradient coverage, curvature changes during large edits, and the cost of repeated HVPs. It also identifies computational and evaluation challenges in extending GIFs to generative model editing.

\item {\bf Theory assumptions and proofs}
    \item[] Question: For each theoretical result, does the paper provide the full set of assumptions and a complete (and correct) proof?
    \item[] Answer: \answerYes{}
    \item[] Justification: The theorem statements specify the assumptions used in the analysis, including the local smoothness and strong-convexity conditions for the GIF error bound. Proofs of the error bound, self-loss response, p-LiSSA convergence, and p-LiSSA unbiasedness are provided in Appendices~\ref{appendix:proof_error_bound}, \ref{appendix:proof_self_loss_bound}, \ref{supple:Proof of Theorem GIF approximation}, and~\ref{sec:proof_unbiased_plissa}.

    \item {\bf Experimental result reproducibility}
    \item[] Question: Does the paper fully disclose all the information needed to reproduce the main experimental results of the paper to the extent that it affects the main claims and/or conclusions of the paper (regardless of whether the code and data are provided or not)?
    \item[] Answer: \answerYes{}
    \item[] Justification: Sec.~\ref{sec:Numerical Experiments} specifies the experimental tasks, datasets, baselines, evaluation metrics, parameter-selection setting, and seed aggregation used for the main results. Appendix~\ref{appendix:reproducibility} further reports the training protocols, model architectures, data-split procedure, hyperparameters, checkpointing, retraining reference models, hardware, software, and per-run timing/memory measurements needed to reproduce the experiments.

\item {\bf Open access to data and code}
    \item[] Question: Does the paper provide open access to the data and code, with sufficient instructions to faithfully reproduce the main experimental results, as described in supplemental material?
    \item[] Answer: \answerYes{}
    \item[] Justification: The implementation and experiment scripts are publicly available at \url{https://github.com/hslyu/GIF}. Appendix~\ref{appendix:reproducibility} describes the public dataset sources, model configurations, optimization hyperparameters, and compute environment used in the experiments.

\item {\bf Experimental setting/details}
    \item[] Question: Does the paper specify all the training and test details (e.g., data splits, hyperparameters, how they were chosen, type of optimizer) necessary to understand the results?
    \item[] Answer: \answerYes{}
    \item[] Justification: Appendix~\ref{appendix:reproducibility} provides the experimental settings needed to understand and reproduce the results, including the training protocol, validation-based model selection, retraining procedure for class-removal references, dataset split pipeline, model architectures, optimizers, learning rates, weight decay, batch sizes, dropout, sequence length, gradient clipping, and number of epochs. Sec.~\ref{sec:Numerical Experiments} additionally defines the evaluation tasks, baselines, metrics, parameter-selection ratios, and seed aggregation used in the reported results.

\item {\bf Experiment statistical significance}
    \item[] Question: Does the paper report error bars suitably and correctly defined or other appropriate information about the statistical significance of the experiments?
    \item[] Answer: \answerYes{}
    \item[] Justification: Sec.~\ref{sec:Numerical Experiments} states that reported means and error bands are computed over 20 random seeds, with table $\pm$ values and shaded figure regions denoting sample standard deviations. Table~\ref{tab:model_edit_comparison} reports mean $\pm$ standard deviation for the model-editing results, and Figs.~\ref{fig:selection_scheme} and~\ref{fig:solver_mnist} report seed-averaged curves with variability bands.

\item {\bf Experiments compute resources}
    \item[] Question: For each experiment, does the paper provide sufficient information on the computer resources (type of compute workers, memory, time of execution) needed to reproduce the experiments?
    \item[] Answer: \answerYes{}
    \item[] Justification: Appendix~\ref{appendix:reproducibility} reports the workstation CPU, system memory, GPU type, GPU memory, NVIDIA driver, and CUDA version used for the numerical experiments. Sec.~\ref{sec:Numerical Experiments} and Fig.~\ref{fig:solver_mnist} additionally report elapsed time and peak memory for the iHVP solver experiments.
    
\item {\bf Code of ethics}
    \item[] Question: Does the research conducted in the paper conform, in every respect, with the NeurIPS Code of Ethics \url{https://neurips.cc/public/EthicsGuidelines}?
    \item[] Answer: \answerYes{}
    \item[] Justification: The experiments use publicly available benchmark datasets and do not involve new data collection, crowdsourcing, human-subject experiments, or deployment in high-stakes settings. Dataset sources, usage conditions, and reproducibility details are documented in Appendix~\ref{appendix:reproducibility}.

\item {\bf Broader impacts}
    \item[] Question: Does the paper discuss both potential positive societal impacts and negative societal impacts of the work performed?
    \item[] Answer: \answerYes{}
    \item[] Justification: The work may have positive impact by making model editing and machine unlearning more reliable and reproducible, which can support responsible correction or removal of learned behavior. The paper is methodological and limits its empirical claims to public benchmark datasets and controlled class-removal experiments rather than deployed high-stakes systems, so we do not identify direct negative societal impacts beyond the general limitations of model-editing research.
    
\item {\bf Safeguards}
    \item[] Question: Does the paper describe safeguards that have been put in place for responsible release of data or models that have a high risk for misuse (e.g., pre-trained language models, image generators, or scraped datasets)?
    \item[] Answer: \answerNA{}
    \item[] Justification: The paper does not release high-risk pretrained models, image generators, scraped datasets, or other artifacts that require special access controls. The planned public release consists of research code and instructions for experiments on publicly available benchmark datasets.

\item {\bf Licenses for existing assets}
    \item[] Question: Are the creators or original owners of assets (e.g., code, data, models), used in the paper, properly credited and are the license and terms of use explicitly mentioned and properly respected?
    \item[] Answer: \answerYes{}
    \item[] Justification: Appendix~\ref{appendix:reproducibility} cites the dataset and checkpoint creators, lists the sources, and summarizes available licenses and usage conditions. It records SVHN's non-commercial use restriction, the PubMed terms, and the absence of an explicit dataset license on the original CIFAR-10 and 20 Newsgroups distribution pages.

\item {\bf New assets}
    \item[] Question: Are new assets introduced in the paper well documented and is the documentation provided alongside the assets?
    \item[] Answer: \answerNA{}
    \item[] Justification: The paper does not introduce new datasets, pretrained models, or benchmarks. The experiments use existing public datasets and the public Hugging Face checkpoint identified in Appendix~\ref{appendix:reproducibility}.

\item {\bf Crowdsourcing and research with human subjects}
    \item[] Question: For crowdsourcing experiments and research with human subjects, does the paper include the full text of instructions given to participants and screenshots, if applicable, as well as details about compensation (if any)? 
    \item[] Answer: \answerNA{}
    \item[] Justification: The paper does not involve crowdsourcing, participant recruitment, user studies, annotation labor, or other research with human subjects.

\item {\bf Institutional review board (IRB) approvals or equivalent for research with human subjects}
    \item[] Question: Does the paper describe potential risks incurred by study participants, whether such risks were disclosed to the subjects, and whether Institutional Review Board (IRB) approvals (or an equivalent approval/review based on the requirements of your country or institution) were obtained?
    \item[] Answer: \answerNA{}
    \item[] Justification: The paper does not involve crowdsourcing or research with human subjects, so IRB approval or equivalent human-subjects review is not applicable.

\item {\bf Declaration of LLM usage}
    \item[] Question: Does the paper describe the usage of LLMs if it is an important, original, or non-standard component of the core methods in this research? Note that if the LLM is used only for writing, editing, or formatting purposes and does \emph{not} impact the core methodology, scientific rigor, or originality of the research, declaration is not required.
    \item[] Answer: \answerNA{}
    \item[] Justification: LLMs were not used as an important, original, or non-standard component of the core research method. They were used only for English polishing, sentence refinement, LaTeX cleanup, reference sanity checks, and assistance with implementing some baseline code, all of which were reviewed and validated by the authors.

\end{enumerate}

\fi
\end{document}